\documentclass[journal]{IEEEtran}

\usepackage[utf8]{inputenc}
\usepackage{color}
\usepackage{xcolor}
\usepackage{array}
\usepackage{verbatim}
\usepackage{float}
\usepackage{amsmath}
\usepackage{amsthm}
\usepackage{amssymb}
\usepackage{graphicx}
\usepackage{longtable}
\usepackage{multirow}
\usepackage[unicode=true,
bookmarks=false,
breaklinks=false,pdfborder={0 0 1},colorlinks=false]
{hyperref}
\hypersetup{
	colorlinks,bookmarksopen,bookmarksnumbered,citecolor=blue,urlcolor=blue}
\usepackage{cite}

\floatstyle{ruled}
\newfloat{algorithm}{tbp}{loa}
\providecommand{\algorithmname}{Algorithm}
\floatname{algorithm}{\protect\algorithmname}

\makeatletter
\let\oldforeign@language\foreign@language
\DeclareRobustCommand{\foreign@language}[1]{%
	\lowercase{\oldforeign@language{#1}}}

\let\oldforeign@language\foreign@language
\DeclareRobustCommand{\foreign@language}[1]{%
	\lowercase{\oldforeign@language{#1}}}

\ifCLASSINFOpdf
\else
\fi

\newcommand{\MYfooter}{\smash{
		\hfil\parbox[t][\height][t]{\textwidth}{\centering
			\thepage}\hfil\hbox{}}}

\def\ps@IEEEtitlepagestyle{%
	\def\@oddhead{\parbox[t][\height][t]{\textwidth}{\centering \scriptsize
			Personal use of this material is permitted. Permission from the author(s) and/or copyright holder(s), must be obtained for all other uses. Please contact us and provide details if you believe this document breaches copyrights.\\
			\noindent\makebox[\linewidth]{}
		}\hfil\hbox{}}%
	\def\@evenhead{\scriptsize\thepage \hfil \leftmark\mbox{}}%
	\def\@oddfoot{\parbox[t][\height][l]{\textwidth}{
			\vspace{-20pt}{\rule{\textwidth}{0.4pt}}\\ \footnotesize\underline{To cite this article:}
			{\bf{\footnotesize\textcolor{red}{H. A. Hashim and A. E. E. Eltoukhy "Nonlinear Filter for Simultaneous Localization and Mapping on a Matrix Lie Group using IMU and Feature Measurements," IEEE Transactions on Systems, Man, and Cybernetics: Systems, vol. 56, no. 4, pp. 2098-2109, 2022.}}} doi: \href{https://doi.org/10.1109/TSMC.2020.3047338}{10.1109/TSMC.2020.3047338}\\
			\noindent\makebox[\linewidth]
		}\hfil\hbox{}}%
	\def\@evenfoot{\MYfooter}}

\makeatother
\newtheorem{defn}{Definition}

\newtheorem{lem}{Lemma}

\newtheorem{thm}{Theorem}
\newtheorem{rem}{Remark}

\newtheorem{assum}{Assumption}

\begin{document}
	\bstctlcite{IEEEexample:BSTcontrol}

	\title{Nonlinear Filter for Simultaneous Localization and Mapping on a Matrix Lie Group using IMU and Feature Measurements}

\author{Hashim A. Hashim$^*$\IEEEmembership{~Member, IEEE} and Abdelrahman E. E. Eltoukhy
	\thanks{This work was supported in part by Thompson Rivers University Internal research fund \# 102315.}
	\thanks{$^*$Corresponding author, H. A. Hashim is with the Department of Engineering and Applied Science, Thompson Rivers University, Kamloops, British Columbia, Canada, V2C-0C8, e-mail: hhashim@tru.ca}
	\thanks{A. E. E. Eltoukhy is with the Department of Industrial and Systems Engineering, The Hong Kong Polytechnic University, Hung Hum, 
	Hong Kong e-mail: abdelrahman.eltoukhy@polyu.edu.hk}
}

%

\maketitle

\begin{abstract}
Simultaneous Localization and Mapping (SLAM) is a process of concurrent
estimation of the vehicle's pose and feature locations with respect
to a frame of reference. This paper proposes a computationally cheap
geometric nonlinear SLAM filter algorithm structured to mimic the
nonlinear motion dynamics of the true SLAM problem posed on the matrix
Lie group of $\mathbb{SLAM}_{n}\left(3\right)$. The nonlinear filter
on manifold is proposed in continuous form and it utilizes available
measurements obtained from group velocity vectors, feature measurements
and an inertial measurement unit (IMU). The unknown bias attached
to velocity measurements is successfully handled by the proposed estimator.
Simulation results illustrate the robustness of the proposed filter
in discrete form demonstrating its utility for the six-degrees-of-freedom
(6 DoF) pose estimation as well as feature estimation in three-dimensional
(3D) space. In addition, the quaternion representation of the nonlinear
filter for SLAM is provided.
\end{abstract}


\IEEEpeerreviewmaketitle{}

\section{Introduction}

\IEEEPARstart{S}{imultaneous} localization and mapping (SLAM) is a critical task that
consists of building a map of an unknown environment while simultaneously
pinpointing the unknown pose (\textit{i.e}, attitude and position)
of the vehicle in three-dimensional (3D) space. SLAM comes into view
when absolute positioning systems, such as global positioning systems
(GPS), are impracticable. It is particularly relevant for applications
performed indoors, underwater, or under harsh weather conditions.
Amongst others, household cleaning devices, security surveillance,
mine exploration, pipelines, location of missing terrestrial and underwater
vehicles, reef monitoring, terrain mapping are all examples of applications
where accurate SLAM is of the essence. Prior knowledge of vehicle
pose, the problem of environment estimation is commonly defined as
a mapping problem which is well-researched by the computer science
and robotics communities \cite{thrun2002robotic}. The reverse problem,
of defining vehicle pose within an established map, is referred to
as pose estimation which has been comprehensively investigated by
the robotics and control community \cite{hashim2019SE3Det,zlotnik2018higher,hashim2020SE3Stochastic}.
SLAM, in turn, constitutes a challenging process of concurrent estimation
of unknown vehicle pose and unknown environment. SLAM problem can
be tackled taking advantage of a set of measurements available with
respect to the body-fixed frame of the moving vehicle. Owing to measurement
contamination with uncertain components, robust filters designed specifically
for the SLAM problem become crucial. Therefore, SLAM has been one
of the core robotics problems for the last three decades and has been
widely explored, for instance \cite{choset2000SLAM,durrant2006simultaneous,bekris2006evaluation,davison2007monoslam,zlotnik2018SLAM,liu2018brain,maurovic2017path,hashim2020SLAMLetter,yuan2019multisensor}.

In robotics, the SLAM problem is traditionally approached using either
a Gaussian or a nonlinear filter. For over a decade, Gaussian approach
was preferred. Several SLAM algorithms were developed on the basis
of Gaussian filters to estimate vehicle state along with the surrounding
features within the map taking uncertainty into consideration. Examples
of Gaussian filters developed for the SLAM problem include FastSLAM
using scalable approach \cite{montemerlo2007fastslam}, unscented
Kalman filter for visual MonoSLAM \cite{davison2007monoslam}, incremental
SLAM \cite{kaess2008isam}, extended Kalman filter (EKF) \cite{huang2007convergence},
neuro-fuzzy EKF \cite{chatterjee2007neuro}, invariant EKF \cite{zhang2017EKF_SLAM},
and others. All of the above approaches are posed in a probabilistic
framework. However, it is important to note that the SLAM problem
is multi-faceted. Commonly addressed aspects of the SLAM problem include
consistency \cite{dissanayake2011review}, high cost computational
complexity \cite{cadena2016past}, poor scalability, environment with
non-static features, and others. Moreover, when approaching the SLAM
problem it is critical to consider: 1) the high complexity of the
pose estimation concerned with vehicles moving in 3D space, 2) the
duality of the problem as it entails both pose and map estimation,
and ultimately 3) its high nonlinearity. In the light of the above
three provisions, firstly, the true SLAM motion dynamics encompass
both vehicle pose and feature dynamics. Secondly, the pose dynamics
of a vehicle moving in 3D space are highly nonlinear, and therefore
are best modeled on the Lie group of the Special Euclidean Group $\mathbb{SE}\left(3\right)$.
And lastly, feature dynamics rely on the vehicle's orientation defined
with respect to the Special Orthogonal Group $\mathbb{SO}\left(3\right)$.
Consequently, owing to the fact that Gaussian filters are based on
linear approximation and are not an optimal fit for the inherently
nonlinear SLAM estimation problem. Nonlinear filters, on the other
hand, can be developed to mimic the true nature of the SLAM problem. 

Taking into consideration the nonlinear nature of the attitude and
pose dynamics, over the past decade, several nonlinear attitude filters
evolved directly on the Lie group of $\mathbb{SO}\left(3\right)$
\cite{lee2012exponential,grip2012attitude,hashim2018SO3Stochastic,hashim2019SO3Wiley},
and pose filters on $\mathbb{SE}\left(3\right)$ \cite{hashim2019SE3Det,zlotnik2018higher,hashim2020SE3Stochastic}
have been proposed. This opened the way for the investigation of the
utility of the Lie group of $\mathbb{SE}\left(3\right)$ for the true
SLAM problem \cite{strasdat2012local}. In recent years, several researchers
have explored nonlinear filters in application to the SLAM problem.
The filter proposed in \cite{johansen2016globally} takes a two-stage
approach, where the first stage consists of vehicle pose estimation
by the means of a nonlinear filter, while the second stage implements
a Kalman filter to obtain feature estimates. The main shortcomings
of the above-mentioned filter are the complexity of having two stages
and inability to explicitly capture the true nonlinear nature of the
SLAM problem. A more recent study proposed nonlinear observers for
SLAM on the matrix Lie group that utilize feature and group velocity
vector measurements directly \cite{zlotnik2018SLAM,hashim2020SLAMLetter}.

Motivated by the previous attempts to capture the complex nature of
the SLAM problem, this work is rooted in the natural nonlinearity
of SLAM and the fact that for $n$ features, SLAM problem is best
modeled on the Lie group of $\mathbb{SLAM}_{n}\left(3\right)$. Taking
advantage of the group velocity vector measurements, availability
of $n$ features, and presence of an inertial measurement unit (IMU),
the contributions of this work are as follows: 
\begin{enumerate}
	\item[1)] A computationally cheap geometric nonlinear deterministic filter
	for SLAM evolved directly on the Lie group of $\mathbb{SLAM}_{n}\left(3\right)$
	and explicitly mimicking the true nature of nonlinear SLAM problem
	is proposed, unlike \cite{johansen2016globally}.
	\item[2)] The nonlinear filter effectively tackles the unknown bias attached
	to the group velocity vector.
	\item[3)] The proposed filter includes gain mapping which allows for cross
	coupling between the innovation of pose and features.
	\item[4)] The presented filter provides asymptotic convergence of the error
	components in the Lyapunov function candidate.
	\item[5)] The error function associated with attitude is guaranteed to be asymptotically
	stable from almost any initial condition, unlike \cite{zlotnik2018SLAM,hashim2020SLAMLetter}.
	\item[6)] A comparison with respect to the previously proposed SLAM filter
	on the Lie group of $\mathbb{SLAM}_{n}\left(3\right)$ is presented.
\end{enumerate}
The remainder of the paper is organized as follows: Section \ref{sec:Preliminaries-and-Math}
presents preliminaries and mathematical notation, the Lie group of
$\mathbb{SO}\left(3\right)$, $\mathbb{SE}\left(3\right)$, and $\mathbb{SLAM}_{n}\left(3\right)$.
Section \ref{sec:SE3_Problem-Formulation} details the SLAM problem,
the true motion kinematics and available measurements. Section \ref{sec:SLAM_Filter}
contains a general nonlinear SLAM filter design followed by the novel
design of the proposed nonlinear filter on $\mathbb{SLAM}_{n}\left(3\right)$.
Section \ref{sec:SE3_Simulations} reveals the effectiveness and robustness
of the proposed filter. Finally, Section \ref{sec:SE3_Conclusion}
summarizes the work.

\section{Preliminaries and Math Notation \label{sec:Preliminaries-and-Math}}

In this paper $\left\{ \mathcal{I}\right\} $ denotes fixed inertial-frame
and $\left\{ \mathcal{B}\right\} $ denotes body-frame fixed to the
moving vehicle. The set of real numbers is denoted by $\mathbb{R}$,
the set of nonnegative real numbers is denoted by $\mathbb{R}_{+}$,
while a $p$-by-$q$ real dimensional space is indicated by $\mathbb{R}^{p\times q}$.
$\mathbf{I}_{p}$ refers to a $p$-by-$p$ identity matrix, $\underline{\mathbf{0}}_{p}$
denotes a zero column vector, and $\left\Vert y\right\Vert =\sqrt{y^{\top}y}$
stands for an Euclidean norm for all $y\in\mathbb{R}^{p}$. $\mathbb{O}\left(3\right)$
represents an orthogonal group that is distinguished by smooth inversion
and multiplication such that
\[
\mathbb{O}\left(3\right)=\left\{ \left.A\in\mathbb{R}^{3\times3}\right|A^{\top}A=AA^{\top}=\mathbf{I}_{3}\right\} 
\]
where $\mathbf{I}_{3}\in\mathbb{R}^{3\times3}$ denotes an identity
matrix. $\mathbb{SO}\left(3\right)$ is a short-hand notation for
the Special Orthogonal Group, a subgroup of $\mathbb{O}\left(3\right)$,
defined as \cite{hashim2018SO3Stochastic,hashim2019SO3Wiley}
\[
\mathbb{SO}\left(3\right)=\left\{ \left.R\in\mathbb{R}^{3\times3}\right|RR^{\top}=R^{\top}R=\mathbf{I}_{3}\text{, }{\rm det}\left(R\right)=+1\right\} 
\]
with ${\rm det\left(\cdot\right)}$ indicating a determinant, and
$R\in\mathbb{SO}\left(3\right)$ denoting orientation, frequently
termed attitude, of a rigid-body in $\left\{ \mathcal{B}\right\} $.
$\mathbb{SE}\left(3\right)$ denotes the Special Euclidean Group defined
by \cite{hashim2020SE3Stochastic}
\[
\mathbb{SE}\left(3\right)=\left\{ \left.\boldsymbol{T}=\left[\begin{array}{cc}
R & P\\
\underline{\mathbf{0}}_{3}^{\top} & 1
\end{array}\right]\in\mathbb{R}^{4\times4}\right|R\in\mathbb{SO}\left(3\right),P\in\mathbb{R}^{3}\right\} 
\]
where $P\in\mathbb{R}^{3}$ denotes position, $R\in\mathbb{SO}\left(3\right)$
denotes orientation, and
\begin{equation}
\boldsymbol{T}=\left[\begin{array}{cc}
R & P\\
\underline{\mathbf{0}}_{3}^{\top} & 1
\end{array}\right]\in\mathbb{SE}\left(3\right)\label{eq:T_SLAM}
\end{equation}
denotes a homogeneous transformation matrix, commonly known as pose,
with $\underline{\mathbf{0}}_{3}$ being a zero column vector. The
term $\boldsymbol{T}$ incorporates the definitions of the rigid-body's
position and orientation in 3D space. The Lie-algebra associated with
$\mathbb{SO}\left(3\right)$ is defined by
\[
\mathfrak{so}\left(3\right)=\left\{ \left.\left[y\right]_{\times}\in\mathbb{R}^{3\times3}\right|\left[y\right]_{\times}^{\top}=-\left[y\right]_{\times},y\in\mathbb{R}^{3}\right\} 
\]
where $\left[y\right]_{\times}$ denotes a skew symmetric matrix and
its map $\left[\cdot\right]_{\times}:\mathbb{R}^{3}\rightarrow\mathfrak{so}\left(3\right)$
as below 
\[
\left[y\right]_{\times}=\left[\begin{array}{ccc}
0 & -y_{3} & y_{2}\\
y_{3} & 0 & -y_{1}\\
-y_{2} & y_{1} & 0
\end{array}\right]\in\mathfrak{so}\left(3\right),\hspace{1em}y=\left[\begin{array}{c}
y_{1}\\
y_{2}\\
y_{3}
\end{array}\right]
\]
Also, for $y,x\in\mathbb{R}^{3}$ one has $\left[y\right]_{\times}x=y\times x$
where $\times$ is a cross product. Analogously to $\mathfrak{so}\left(3\right)$,
let us represent the $\mathbb{SE}\left(3\right)$ Lie-algebra with
$\mathfrak{se}\left(3\right)$ defined by{\small{}
	\[
	\mathfrak{se}\left(3\right)=\left\{ \left[U\right]_{\wedge}\in\mathbb{R}^{4\times4}\left|\exists\Omega,V\in\mathbb{R}^{3}:\left[U\right]_{\wedge}=\left[\begin{array}{cc}
	\left[\Omega\right]_{\times} & V\\
	\underline{\mathbf{0}}_{3}^{\top} & 0
	\end{array}\right]\right.\right\} 
	\]
}where $\left[\cdot\right]_{\wedge}$ denotes a wedge operator and
the wedge map $\left[\cdot\right]_{\wedge}:\mathbb{R}^{6}\rightarrow\mathfrak{se}\left(3\right)$
is 
\[
\left[U\right]_{\wedge}=\left[\begin{array}{cc}
\left[\Omega\right]_{\times} & V\\
\underline{\mathbf{0}}_{3}^{\top} & 0
\end{array}\right]\in\mathfrak{se}\left(3\right),\hspace{1em}U=\left[\begin{array}{c}
\Omega\\
V
\end{array}\right]\in\mathbb{R}^{6}
\]
$\mathbf{vex}:\mathfrak{so}\left(3\right)\rightarrow\mathbb{R}^{3}$
defines the inverse mapping of $\left[\cdot\right]_{\times}$ where
\begin{equation}
\mathbf{vex}\left(\left[y\right]_{\times}\right)=y,\hspace{1em}\forall y\in\mathbb{R}^{3}\label{eq:SLAM_VEX}
\end{equation}
Let $\boldsymbol{\mathcal{P}}_{a}$ define the anti-symmetric projection
on the $\mathfrak{so}\left(3\right)$ Lie-algebra:
\begin{equation}
\boldsymbol{\mathcal{P}}_{a}\left(A\right)=\frac{1}{2}(A-A^{\top})\in\mathfrak{so}\left(3\right),\hspace{1em}\forall A\in\mathbb{R}^{3\times3}\label{eq:SLAM_Pa}
\end{equation}
Additionally, let $\boldsymbol{\Upsilon}\left(\cdot\right)$ stand
for the composition mapping $\boldsymbol{\Upsilon}=\mathbf{vex}\circ\boldsymbol{\mathcal{P}}_{a}$
such that
\begin{equation}
\boldsymbol{\Upsilon}\left(A\right)=\mathbf{vex}\left(\boldsymbol{\mathcal{P}}_{a}\left(A\right)\right)\in\mathbb{R}^{3},\hspace{1em}\forall A\in\mathbb{R}^{3\times3}\label{eq:SLAM_VEX_a}
\end{equation}
$\left\Vert R\right\Vert _{{\rm I}}$ defines the Euclidean distance
of $R\in\mathbb{SO}\left(3\right)$ such that
\begin{equation}
\left\Vert R\right\Vert _{{\rm I}}=\frac{1}{4}{\rm Tr}\{\mathbf{I}_{3}-R\}\in\left[0,1\right]\label{eq:SLAM_Ecul_Dist}
\end{equation}
For any $\boldsymbol{T}\in\mathbb{SE}\left(3\right)$ and $U\in\mathbb{R}^{6}$
given that $\left[U\right]_{\wedge}\in\mathfrak{se}\left(3\right)$,
the adjoint map ${\rm Ad}_{\boldsymbol{T}}:\mathbb{SE}\left(3\right)\times\mathfrak{se}\left(3\right)\rightarrow\mathfrak{se}\left(3\right)$
is defined by
\begin{equation}
{\rm Ad}_{\boldsymbol{T}}\left(\left[U\right]_{\wedge}\right)=\boldsymbol{T}\left[U\right]_{\wedge}\boldsymbol{T}^{-1}\in\mathfrak{se}\left(3\right)\label{eq:SLAM_Adjoint}
\end{equation}
For any homogeneous transformation matrix $\boldsymbol{T}\in\mathbb{SE}\left(3\right)$,
for instance \eqref{eq:T_SLAM}, define an augmented adjoint map $\overline{{\rm Ad}}_{\boldsymbol{T}}:\mathbb{SE}\left(3\right)\rightarrow\mathbb{R}^{6\times6}$
\begin{equation}
\overline{{\rm Ad}}_{\boldsymbol{T}}=\left[\begin{array}{cc}
R & 0_{3\times3}\\
\left[P\right]_{\times}R & R
\end{array}\right]\in\mathbb{R}^{6\times6}\label{eq:SLAM_Adjoint_Aug}
\end{equation}
Thus, from \eqref{eq:SLAM_Adjoint} and \eqref{eq:SLAM_Adjoint_Aug}
it follows that
\begin{equation}
{\rm Ad}_{\boldsymbol{T}}\left(\left[U\right]_{\wedge}\right)=\left[\,\overline{{\rm Ad}}_{\boldsymbol{T}}U\right]_{\wedge},\hspace{1em}\boldsymbol{T}\in\mathbb{SE}\left(3\right),U\in\mathbb{R}^{6}\label{eq:SLAM_Adjoint_MAP}
\end{equation}
Let $\overset{\circ}{\mathcal{M}}$ and $\overline{\mathcal{M}}$
be submanifolds of $\mathbb{R}^{4}$ such that
\begin{align*}
\overset{\circ}{\mathcal{M}} & =\left\{ \left.\overset{\circ}{x}=\left[\begin{array}{cc}
x^{\top} & 0\end{array}\right]^{\top}\in\mathbb{R}^{4}\right|x\in\mathbb{R}^{3}\right\} \\
\overline{\mathcal{M}} & =\left\{ \left.\overline{x}=\left[\begin{array}{cc}
x^{\top} & 1\end{array}\right]^{\top}\in\mathbb{R}^{4}\right|x\in\mathbb{R}^{3}\right\} 
\end{align*}
Define the Lie group $\mathbb{SLAM}_{n}\left(3\right)=\mathbb{SE}\left(3\right)\times\overline{\mathcal{M}}^{n}$
such that
\[
\mathbb{SLAM}_{n}\left(3\right)=\left\{ X=(\boldsymbol{T},\overline{{\rm p}})\left|\boldsymbol{T}\in\mathbb{SE}\left(3\right),\overline{{\rm p}}\in\overline{\mathcal{M}}^{n}\right.\right\} 
\]
where $\overline{{\rm p}}=[\overline{{\rm p}}_{1},\overline{{\rm p}}_{2},\ldots,\overline{{\rm p}}_{n}]\in\overline{\mathcal{M}}^{n}$
and $\overline{\mathcal{M}}^{n}=\overline{\mathcal{M}}\times\overline{\mathcal{M}}\times\cdots\times\overline{\mathcal{M}}$.
The Lie algebra of $\mathbb{SLAM}_{n}\left(3\right)$ which is the
tangent space at the identity element of $X=(\boldsymbol{T},\overline{{\rm p}})\in\mathbb{SLAM}_{n}\left(3\right)$
is denoted by $\mathfrak{slam}_{n}\left(3\right)=\mathfrak{se}\left(3\right)\times\overset{\circ}{\mathcal{M}}^{n}$
and defined by
\[
\mathfrak{slam}_{n}\left(3\right)=\left\{ \mathcal{Y}=(\left[U\right]_{\wedge},\overset{\circ}{{\rm v}})\left|\left[U\right]_{\wedge}\in\mathfrak{se}\left(3\right),\overset{\circ}{{\rm v}}\in\overset{\circ}{\mathcal{M}}^{n}\right.\right\} 
\]
where $\overset{\circ}{{\rm v}}=[\overset{\circ}{{\rm v}}_{1},\overset{\circ}{{\rm v}}_{2},\ldots,\overset{\circ}{{\rm v}}_{n}]\in\overset{\circ}{\mathcal{M}}^{n}$
and $\overset{\circ}{\mathcal{M}}^{n}=\overset{\circ}{\mathcal{M}}\times\overset{\circ}{\mathcal{M}}\times\cdots\times\overset{\circ}{\mathcal{M}}$
such that $\overset{\circ}{{\rm v}}_{i}=\left[{\rm v}_{i}^{\top},0\right]^{\top}\in\overset{\circ}{\mathcal{M}}\forall i=1,\ldots,n$.
The identities below will be used in the forthcoming derivations 
\begin{align}
\left[Ry\right]_{\times}= & R\left[y\right]_{\times}R^{\top},\hspace{1em}y\in{\rm \mathbb{R}}^{3},R\in\mathbb{SO}\left(3\right)\label{eq:SLAM_Identity1}\\
\left[y\times x\right]_{\times}= & xy^{\top}-yx^{\top},\hspace{1em}x,y\in{\rm \mathbb{R}}^{3}\label{eq:SLAM_Identity2}\\
{\rm Tr}\left\{ M\left[y\right]_{\times}\right\} = & {\rm Tr}\left\{ \boldsymbol{\mathcal{P}}_{a}\left(M\right)\left[y\right]_{\times}\right\} ,\hspace{1em}y\in{\rm \mathbb{R}}^{3},M\in\mathbb{R}^{3\times3}\nonumber \\
= & -2\mathbf{vex}\left(\boldsymbol{\mathcal{P}}_{a}\left(M\right)\right)^{\top}y\label{eq:SLAM_Identity6}
\end{align}

\section{SLAM Kinematics and Measurements\label{sec:SE3_Problem-Formulation}}

The complexity of the SLAM consists in the concurrent estimation of
two unknowns: 1) vehicle pose (orientation and position) $\boldsymbol{T}\in\mathbb{SE}\left(3\right)$,
and 2) position of the features within the environment $\overline{{\rm p}}=[\overline{{\rm p}}_{1},\overline{{\rm p}}_{2},\ldots,\overline{{\rm p}}_{n}]\in\overline{\mathcal{M}}^{n}$.
As such, given a set of measurements, SLAM estimation process is comprised
of 1) vehicle pose estimation relative to the features within the
map and simultaneous 2) estimation of the map (positioning of $\overline{{\rm p}}$
within the map). Fig. \ref{fig:SLAM} presents a conceptual representation
of the SLAM problem. 
\begin{figure*}
	\centering{}\includegraphics[scale=0.6]{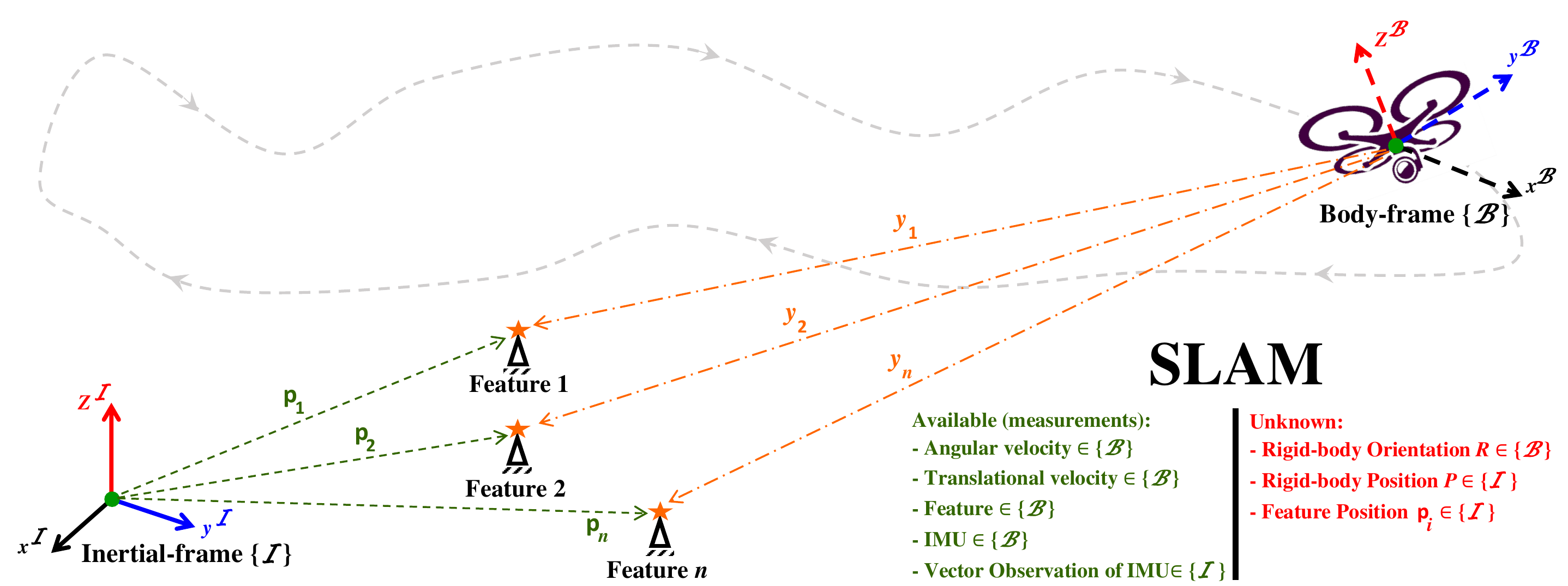}\caption{SLAM estimation problem.}
	\label{fig:SLAM}
\end{figure*}

Let $R\in\mathbb{SO}\left(3\right)$ be the orientation of a rigid-body
and $P\in\mathbb{R}^{3}$ be its translation into 3D space where $R\in\left\{ \mathcal{B}\right\} $
and $P\in\left\{ \mathcal{I}\right\} $. Assume that the map has $n$
features with ${\rm p}_{i}$ being the $i$th feature location for
all $i=1,2,\ldots,n$, and ${\rm p}_{i}\in\left\{ \mathcal{I}\right\} $.
Let $X=(\boldsymbol{T},\overline{{\rm p}})\in\mathbb{SLAM}_{n}\left(3\right)$
represent the true pose of the rigid-body similar to \eqref{eq:T_SLAM}
and features $\overline{{\rm p}}=[\overline{{\rm p}}_{1},\overline{{\rm p}}_{2},\ldots,\overline{{\rm p}}_{n}]\in\overline{\mathcal{M}}^{n}$
where $X$ is unknown. Let $\mathcal{Y}=(\left[U\right]_{\wedge},\overset{\circ}{{\rm v}})\in\mathfrak{slam}_{n}\left(3\right)$
represent the true group velocity that is continuous and bounded such
that $\overset{\circ}{{\rm v}}=[\overset{\circ}{{\rm v}}_{1},\overset{\circ}{{\rm v}}_{2},\ldots,\overset{\circ}{{\rm v}}_{n}]\in\overset{\circ}{\mathcal{M}}^{n}$,
and assume that $\mathcal{Y}$ measurements are readily available.
Therefore, from \eqref{eq:T_SLAM}, the true motion dynamics of the
rigid-body pose and $n$-features can be expressed as
\begin{equation}
\begin{cases}
\dot{\boldsymbol{T}} & =\boldsymbol{T}\left[U\right]_{\wedge}\\
\dot{{\rm p}}_{i} & =R{\rm v}_{i},\hspace{1em}\forall i=1,2,\ldots,n
\end{cases}\label{eq:SLAM_True_dot}
\end{equation}
or, to put simply,
\[
\begin{cases}
\dot{R} & =R\left[\Omega\right]_{\times}\\
\dot{P} & =RV\\
\dot{{\rm p}}_{i} & =R{\rm v}_{i},\hspace{1em}\forall i=1,2,\ldots,n
\end{cases}
\]
where $U=\left[\Omega^{\top},V^{\top}\right]^{\top}\in\mathbb{R}^{6}$
denotes the group velocity vector, $\Omega\in\mathbb{R}^{3}$ and
$V\in\mathbb{R}^{3}$ stand for the true angular and translational
velocity of the rigid-body expressed in the body-frame, respectively,
while ${\rm v}_{i}\in\mathbb{R}^{3}$ represents the $i$th linear
velocity of a feature expressed in the body-frame such that $\Omega,V,{\rm v}_{i}\in\left\{ \mathcal{B}\right\} $.
As has been previously discussed and in accordance with Fig. \ref{fig:SLAM},
$\boldsymbol{T}\in\mathbb{SE}\left(3\right)$ and $\overline{{\rm p}}\in\overline{\mathcal{M}}^{n}$
are unknown. However, the rigid-body (vehicle) is equipped with multiple
sensors that provide us with a set of measurements. The measurements
of angular and translational velocity are given by \cite{hashim2019SE3Det,hashim2020SE3Stochastic}
\begin{equation}
\begin{cases}
\Omega_{m} & =\Omega+b_{\Omega}+n_{\Omega}\in\mathbb{R}^{3}\\
V_{m} & =V+b_{V}+n_{V}\in\mathbb{R}^{3}
\end{cases}\label{eq:SLAM_TVelcoity}
\end{equation}
with $b_{\star}$ and $n_{\star}$ being unknown constant bias and
random noise, respectively, associated with the $\star$ element.
Let $U_{m}=\left[\Omega_{m}^{\top},V_{m}^{\top}\right]^{\top}$, $b_{U}=\left[b_{\Omega}^{\top},b_{V}^{\top}\right]^{\top}$,
and $n_{U}=\left[n_{\Omega}^{\top},n_{V}^{\top}\right]^{\top}$ for
all $U_{m},b_{U},n_{U}\in\mathbb{R}^{6}$ and $U_{m},b_{U},n_{U}\in\left\{ \mathcal{B}\right\} $.
Under the assumption of a static environment adopted in this paper,
$\dot{{\rm p}}_{i}=\underline{\mathbf{0}}_{3}$ and entails that ${\rm v}_{i}=\underline{\mathbf{0}}_{3}$
$\forall i=1,2,\ldots,n$. The body-frame measurements associated
with the orientation determination can be expressed as \cite{hashim2018SO3Stochastic,hashim2019SO3Wiley}
\[
\overset{\circ}{a}_{j}=\boldsymbol{T}^{-1}\overset{\circ}{r}_{j}+\overset{\circ}{b}_{j}^{a}+\overset{\circ}{n}_{j}^{a}\in\overset{\circ}{\mathcal{M}},\hspace{1em}j=1,2,\ldots,n_{R}
\]
or, more simply, 
\begin{equation}
a_{j}=R^{\top}r_{j}+b_{j}^{a}+n_{j}^{a}\in\mathbb{R}^{3}\label{eq:SLAM_Vect_R}
\end{equation}
where $r_{j}$ is the $j$th known inertial-frame vector, $b_{j}^{a}$
is unknown constant bias, and $n_{j}^{a}$ is unknown random noise.
It can be easily found that the inverse of $\boldsymbol{T}$ is $\boldsymbol{T}^{-1}=\left[\begin{array}{cc}
R^{\top} & -R^{\top}P\\
\underline{\mathbf{0}}_{3}^{\top} & 1
\end{array}\right]\in\mathbb{SE}\left(3\right)$. In our analysis, it is assumed that $b_{j}^{a}=n_{j}^{a}=\underline{\mathbf{0}}_{3}$.
Both $r_{j}$ and $a_{j}$ in \eqref{eq:SLAM_Vect_R} can be normalized
and utilized to extract the rigid-body's attitude as
\begin{equation}
\upsilon_{j}^{r}=\frac{r_{j}}{\left\Vert r_{j}\right\Vert },\hspace{1em}\upsilon_{j}^{a}=\frac{a_{j}}{\left\Vert a_{j}\right\Vert }\label{eq:SLAM_Vector_norm}
\end{equation}
Let us group the normalized vectors into the following two sets
\begin{equation}
\begin{cases}
\upsilon^{r} & =[\upsilon_{1}^{r},\upsilon_{2}^{r},\ldots,\upsilon_{n_{R}}^{r}]\in\left\{ \mathcal{I}\right\} \\
\upsilon^{a} & =[\upsilon_{1}^{a},\upsilon_{2}^{a},\ldots,\upsilon_{n_{R}}^{a}]\in\left\{ \mathcal{B}\right\} 
\end{cases}\label{eq:SE3STCH_Set_R_Norm}
\end{equation}

\begin{rem}
	\label{rem:R_Marix}The orientation of a rigid-body can be extracted
	provided that both sets in \eqref{eq:SE3STCH_Set_R_Norm} have a rank
	of 3, indicating that at least two non-collinear vectors in $\left\{ \mathcal{B}\right\} $
	and their observations in $\left\{ \mathcal{I}\right\} $ are obtainable.
	The expression in \eqref{eq:SLAM_Vect_R} exemplifies two measurements
	acquired from a low cost IMU and, while the third data point in both
	$\left\{ \mathcal{B}\right\} $ and $\left\{ \mathcal{I}\right\} $
	can be obtained by means of a cross product $\upsilon_{3}^{a}=\upsilon_{1}^{a}\times\upsilon_{2}^{a}$
	and $\upsilon_{3}^{r}=\upsilon_{1}^{r}\times\upsilon_{2}^{r}$, respectively.
\end{rem}
Obtaining $n$ features in the body-frame can be done through the
utility of low-cost inertial vision units where the $i$th measurement
is defined as
\[
\overline{y}_{i}=\boldsymbol{T}^{-1}\overline{{\rm p}}_{i}+\overset{\circ}{b}_{i}^{y}+\overset{\circ}{n}_{i}^{y}\in\overline{\mathcal{M}},\hspace{1em}\forall i=1,2,\ldots,n
\]
or more simply,
\begin{equation}
y_{i}=R^{\top}({\rm p}_{i}-P)+b_{i}^{y}+n_{i}^{y}\in\mathbb{R}^{3}\label{eq:SLAM_Vec_Landmark}
\end{equation}
where the definitions of $R$, $P$, and ${\rm p}_{i}$ can be found
in \eqref{eq:SLAM_True_dot}, and $b_{i}^{y}$ and $n_{i}^{y}$ are
unknown constant bias and random noise, respectively, for all $y_{i},b_{i}^{y},n_{i}^{y}\in\left\{ \mathcal{B}\right\} $.
In our analysis, it is assumed that $b_{i}^{y}=n_{i}^{y}=\underline{\mathbf{0}}_{3}$.

\begin{assum}\label{Assumption:Feature}Assume that the total number
	of features available for measurement is greater than or equal to
	3 which is a necessity for an unambiguous definition of a plane with
	$\overline{y}=[\overline{y}_{1},\overline{y}_{2},\ldots,\overline{y}_{n}]\in\overline{\mathcal{M}}^{n}$.\end{assum}

\section{Nonlinear Filter Design \label{sec:SLAM_Filter}}

This section presents two nonlinear filter designs for the SLAM problem.
The first nonlinear filter incorporates only the surrounding feature
measurements. The second nonlinear filter design considers measurements
obtained from a typical low cost IMU in addition to the surrounding
feature measurements. Let the estimate of pose be
\[
\hat{\boldsymbol{T}}=\left[\begin{array}{cc}
\hat{R} & \hat{P}\\
\underline{\mathbf{0}}_{3}^{\top} & 1
\end{array}\right]\in\mathbb{SE}\left(3\right)
\]
where $\hat{R}$ and $\hat{P}$ denote estimates of the true orientation
and position, respectively. Let $\hat{{\rm p}}_{i}$ be the estimate
of the true $i$th feature ${\rm p}_{i}$. Define the error between
$\boldsymbol{T}$ and $\hat{\boldsymbol{T}}$ as
\begin{align}
\tilde{\boldsymbol{T}}=\hat{\boldsymbol{T}}\boldsymbol{T}^{-1} & =\left[\begin{array}{cc}
\hat{R} & \hat{P}\\
\underline{\mathbf{0}}_{3}^{\top} & 1
\end{array}\right]\left[\begin{array}{cc}
R^{\top} & -R^{\top}P\\
\underline{\mathbf{0}}_{3}^{\top} & 1
\end{array}\right]\nonumber \\
& =\left[\begin{array}{cc}
\tilde{R} & \tilde{P}\\
\underline{\mathbf{0}}_{3}^{\top} & 1
\end{array}\right]\label{eq:SLAM_T_error}
\end{align}
where $\tilde{R}=\hat{R}R^{\top}$ and $\tilde{P}=\hat{P}-\tilde{R}P$.
The objective of pose estimation is to asymptotically drive $\tilde{\boldsymbol{T}}\rightarrow\mathbf{I}_{4}$
which in turn would cause $\tilde{R}\rightarrow\mathbf{I}_{3}$ and
$\tilde{P}\rightarrow\underline{\mathbf{0}}_{3}$. To this end, define
the error between $\hat{{\rm p}}_{i}$ and ${\rm p}_{i}$ as follows:
\begin{equation}
\overset{\circ}{e}_{i}=\overline{\hat{{\rm p}}}_{i}-\tilde{\boldsymbol{T}}\,\overline{{\rm p}}_{i}\in\overset{\circ}{\mathcal{M}}\label{eq:SLAM_e}
\end{equation}
where $\overline{\hat{{\rm p}}}_{i}=\left[\hat{{\rm p}}_{i}^{\top},1\right]^{\top}\in\overline{\mathcal{M}}$.
In view of \eqref{eq:SLAM_Vec_Landmark}, $\overset{\circ}{e}_{i}=\overline{\hat{{\rm p}}}_{i}-\hat{\boldsymbol{T}}\boldsymbol{T}^{-1}\,\overline{{\rm p}}_{i}$
can be expressed as
\begin{equation}
\overset{\circ}{e}_{i}=\overline{\hat{{\rm p}}}_{i}-\hat{\boldsymbol{T}}\,\overline{y}_{i}=\left[e_{i}^{\top},0\right]^{\top}\label{eq:SLAM_e_Final}
\end{equation}
Thus, it can be found that
\begin{align}
\overset{\circ}{e}_{i} & =\left[\begin{array}{c}
\hat{{\rm p}}_{i}\\
1
\end{array}\right]-\left[\begin{array}{cc}
\hat{R} & \hat{P}\\
\underline{\mathbf{0}}_{3}^{\top} & 1
\end{array}\right]\left[\begin{array}{c}
R^{\top}\left({\rm p}_{i}-P\right)\\
1
\end{array}\right]\nonumber \\
& =\left[\begin{array}{c}
\tilde{{\rm p}}_{i}-\tilde{P}\\
0
\end{array}\right]\label{eq:SLAM_e_tilde}
\end{align}
where $\tilde{{\rm p}}_{i}=\hat{{\rm p}}_{i}-\tilde{R}{\rm p}_{i}$
and $\tilde{P}=\hat{P}-\tilde{R}P$. Consider $n_{\Omega}=n_{V}=\underline{\mathbf{0}}_{3}$
and for the group velocity in \eqref{eq:SLAM_TVelcoity}, let the
estimate of the unknown bias $b_{U}$ be $\hat{b}_{U}=\left[\hat{b}_{\Omega}^{\top},\hat{b}_{V}^{\top}\right]^{\top}$.
Define the error between $b_{U}$ and $\hat{b}_{U}$ as
\begin{equation}
\begin{cases}
\tilde{b}_{\Omega} & =b_{\Omega}-\hat{b}_{\Omega}\\
\tilde{b}_{V} & =b_{V}-\hat{b}_{V}
\end{cases}\label{eq:SLAM_b_error}
\end{equation}
where $\tilde{b}_{U}=\left[\tilde{b}_{\Omega}^{\top},\tilde{b}_{V}^{\top}\right]^{\top}\in\mathbb{R}^{6}$.
Before proceeding, it is important to emphasize that the true SLAM
dynamics in \eqref{eq:SLAM_True_dot} are nonlinear and are modeled
on the Lie group of $\mathbb{SLAM}_{n}\left(3\right)=\mathbb{SE}\left(3\right)\times\overline{\mathcal{M}}^{n}$
with a tangent space of $\mathfrak{slam}_{n}\left(3\right)=\mathfrak{se}\left(3\right)\times\overset{\circ}{\mathcal{M}}^{n}$
where $X=(\boldsymbol{T},\overline{{\rm p}})\in\mathbb{SLAM}_{n}\left(3\right)$
and $\mathcal{Y}=(\left[U\right]_{\wedge},\overset{\circ}{{\rm v}})\in\mathfrak{slam}_{n}\left(3\right)$.
It thus follows logically that an efficient SLAM filter should be
designed to imitate the nonlinearity of the true SLAM problem by modeling
it on the Lie group of $\mathbb{SLAM}_{n}\left(3\right)$ with a tangent
space $\mathfrak{slam}_{n}\left(3\right)$. Accordingly, the proposed
filter has the structure of $\hat{X}=(\hat{\boldsymbol{T}},\overline{\hat{{\rm p}}})\in\mathbb{SLAM}_{n}\left(3\right)$
and $\hat{\mathcal{Y}}=([\hat{U}]_{\wedge},\overset{\circ}{\hat{{\rm v}}})\in\mathfrak{slam}_{n}\left(3\right)$
where $\hat{\boldsymbol{T}}\in\mathbb{SE}\left(3\right)$ and $\overline{\hat{{\rm p}}}=[\overline{\hat{{\rm p}}}_{1},\ldots,\overline{\hat{{\rm p}}}_{n}]\in\overline{\mathcal{M}}^{n}$
are the pose and feature estimates, respectively, while $\hat{U}\in\mathfrak{se}\left(3\right)$
and $\overset{\circ}{\hat{{\rm v}}}=[\overset{\circ}{\hat{{\rm v}}}_{1},\ldots,\overset{\circ}{\hat{{\rm v}}}_{n}]\in\overset{\circ}{\mathcal{M}}^{n}$
are velocities to be designed in the subsequent subsections. Additionally
note that, $\overset{\circ}{\hat{{\rm v}}}_{i}=[\hat{{\rm v}}_{i}^{\top},0]\in\overset{\circ}{\mathcal{M}}$
and $\overline{\hat{{\rm p}}}_{i}=\left[\hat{{\rm p}}_{i}^{\top},1\right]^{\top}\in\overline{\mathcal{M}}$
for all $i=1,2,\cdots,n$.

\subsection{Nonlinear Filter Design without IMU\label{subsec:Det_without_IMU}}

This subsection presents a SLAM nonlinear filter design that operates
based solely on measurements obtained from the surrounding features
along with angular and translational velocities. Consider the following
nonlinear filter evolved directly on $\mathbb{SLAM}_{n}\left(3\right)$:

\begin{align}
\dot{\hat{\boldsymbol{T}}} & =\hat{\boldsymbol{T}}\left[U_{m}-\hat{b}_{U}-W_{U}\right]_{\wedge}\label{eq:SLAM_T_est_dot_f1}\\
W_{U} & =-\sum_{i=1}^{n}k_{w}\overline{{\rm Ad}}_{\hat{\boldsymbol{T}}^{-1}}\left[\begin{array}{c}
\left[\hat{R}y_{i}+\hat{P}\right]_{\times}\\
\mathbf{I}_{3}
\end{array}\right]e_{i}\label{eq:SLAM_W_f1}\\
\dot{\hat{b}}_{U} & =-\sum_{i=1}^{n}\frac{\Gamma}{\alpha_{i}}\overline{{\rm Ad}}_{\hat{\boldsymbol{T}}}^{\top}\left[\begin{array}{c}
\left[\hat{R}y_{i}+\hat{P}\right]_{\times}\\
\mathbf{I}_{3}
\end{array}\right]e_{i}\label{eq:SLAM_b_est_dot_f1}\\
\dot{{\rm \hat{p}}}_{i} & =-k_{1}e_{i},\hspace{1em}i=1,2,\ldots,n\label{eq:SLAM_p_est_dot_f1}
\end{align}
where $k_{w}$, $k_{1}$, $\Gamma$, and $\alpha_{i}$ are positive
constants, $e_{i}$ is as defined in \eqref{eq:SLAM_e_Final}, and
$\overline{{\rm Ad}}_{\hat{\boldsymbol{T}}}=\left[\begin{array}{cc}
\hat{R} & 0_{3\times3}\\{}
[\hat{P}]_{\times}\hat{R} & \hat{R}
\end{array}\right]$ for all $i=1,2,\cdots,n$. Also, $W_{U}=\left[W_{\Omega}^{\top},W_{V}^{\top}\right]^{\top}\in\mathbb{R}^{6}$
is a correction factor and $\hat{b}_{U}=\left[\hat{b}_{\Omega}^{\top},\hat{b}_{V}^{\top}\right]^{\top}\in\mathbb{R}^{6}$
is the estimate of $b_{U}$.
\begin{thm}
	Consider combining the SLAM dynamics $\dot{X}=(\dot{\boldsymbol{T}},\dot{\overline{{\rm p}}})$
	in \eqref{eq:SLAM_True_dot} with feature measurements (output $\overline{y}_{i}=\boldsymbol{T}^{-1}\overline{{\rm p}}_{i}$)
	for all $i=1,2,\ldots,n$ and the velocity measurements ($U_{m}=U+b_{U}$).
	Let Assumption \ref{Assumption:Feature} hold. Let the filter design
	in \eqref{eq:SLAM_T_est_dot_f1}, \eqref{eq:SLAM_W_f1}, \eqref{eq:SLAM_b_est_dot_f1},
	and \eqref{eq:SLAM_p_est_dot_f1} be coupled with the measurements
	of $U_{m}$ and $\overline{y}_{i}$. Consider the design parameters
	$k_{w}$, $k_{1}$, $\Gamma$, and $\alpha_{i}$ to be positive constants
	for all $i=1,2,\ldots,n$, and define the set
	\begin{align}
	\mathcal{S}= & \left\{ (e_{1},e_{2},\ldots,e_{n})\in\mathbb{R}^{3}\times\mathbb{R}^{3}\times\cdots\times\mathbb{R}^{3}\right|\nonumber \\
	& \hspace{10em}\left.e_{i}=\underline{\mathbf{0}}_{3}\forall i=1,2,\ldots,n\right\} \label{eq:SLAM_Set1}
	\end{align}
	Then, 1) the error $e_{i}$ in \eqref{eq:SLAM_e} converges exponentially
	to $\mathcal{S}$, 2) the trajectory of $\tilde{\boldsymbol{T}}$
	remains bounded and 3) there exist constants $R_{c}\in\mathbb{SO}\left(3\right)$
	and $P_{c}\in\mathbb{R}^{3}$ such that $\tilde{R}\rightarrow R_{c}$
	and $\tilde{P}\rightarrow P_{c}$ as $t\rightarrow\infty$.
\end{thm}
\begin{proof}Considering the fact that $\boldsymbol{\dot{T}}^{-1}=-\boldsymbol{T}^{-1}\boldsymbol{\dot{T}}\boldsymbol{T}^{-1}$
	and coupling it with the adjoint map in \eqref{eq:SLAM_Adjoint},
	the error dynamics of $\tilde{\boldsymbol{T}}$ defined in \eqref{eq:SLAM_T_error}
	can be expressed as below
	\begin{align}
	\dot{\tilde{\boldsymbol{T}}} & =\dot{\hat{\boldsymbol{T}}}\boldsymbol{T}^{-1}+\hat{\boldsymbol{T}}\dot{\boldsymbol{T}}^{-1}\nonumber \\
	& =\hat{\boldsymbol{T}}\left[U+\tilde{b}_{U}-W_{U}\right]_{\wedge}\boldsymbol{T}^{-1}-\hat{\boldsymbol{T}}\left[U\right]_{\wedge}\boldsymbol{T}^{-1}\nonumber \\
	& ={\rm Ad}_{\hat{\boldsymbol{T}}}\left(\left[\tilde{b}_{U}-W_{U}\right]_{\wedge}\right)\tilde{\boldsymbol{T}}\label{eq:SLAM_T_error_dot}
	\end{align}
	Hence, the error dynamics of $\overset{\circ}{e}_{i}$ in \eqref{eq:SLAM_e}
	become
	\begin{align}
	\overset{\circ}{\dot{e}}_{i} & =\overset{\circ}{\dot{\hat{{\rm p}}}}_{i}-\dot{\tilde{\boldsymbol{T}}}\,\overline{{\rm p}}_{i}-\tilde{\boldsymbol{T}}\,\dot{\overline{{\rm p}}}_{i}\nonumber \\
	& =\overset{\circ}{\dot{\hat{{\rm p}}}}_{i}-{\rm Ad}_{\hat{\boldsymbol{T}}}\left(\left[\tilde{b}_{U}-W_{U}\right]_{\wedge}\right)\tilde{\boldsymbol{T}}\,\overline{{\rm p}}_{i}\label{eq:SLAM_e_dot}
	\end{align}
	Recalling the adjoint expressions in \eqref{eq:SLAM_Adjoint}, \eqref{eq:SLAM_Adjoint_Aug},
	and \eqref{eq:SLAM_Adjoint_MAP}, one finds
	\begin{align*}
	& {\rm Ad}_{\hat{\boldsymbol{T}}}\left(\left[\tilde{b}_{U}-W_{U}\right]_{\wedge}\right)=\left[\overline{{\rm Ad}}_{\hat{\boldsymbol{T}}}(\tilde{b}_{U}-W_{U})\right]_{\wedge}\\
	& \hspace{5em}=\left[\left[\begin{array}{cc}
	\hat{R} & 0_{3\times3}\\{}
	[\hat{P}]_{\times}\hat{R} & \hat{R}
	\end{array}\right]\left[\begin{array}{c}
	\tilde{b}_{\Omega}-W_{\Omega}\\
	\tilde{b}_{V}-W_{V}
	\end{array}\right]\right]_{\wedge}
	\end{align*}
	According to the above result, one obtains
	\begin{align}
	& {\rm Ad}_{\hat{\boldsymbol{T}}}\left(\left[\tilde{b}_{U}-W_{U}\right]_{\wedge}\right)\tilde{\boldsymbol{T}}\,\overline{{\rm p}}_{i}\nonumber \\
	& \hspace{3em}=\left[\begin{array}{cc}
	\left[\hat{R}y_{i}+\hat{P}\right]_{\times} & \underline{\mathbf{0}}_{3}\\
	\mathbf{I}_{3} & \underline{\mathbf{0}}_{3}
	\end{array}\right]^{\top}\overline{{\rm Ad}}_{\hat{\boldsymbol{T}}}\left(\tilde{b}_{U}-W_{U}\right)\label{eq:SLAM_expression1}
	\end{align}
	Therefore, one can rewrite the expression in \eqref{eq:SLAM_e_dot}
	as
	\begin{align*}
	\overset{\circ}{\dot{e}}_{i} & =\overset{\circ}{\dot{\hat{{\rm p}}}}_{i}-\left[\begin{array}{cc}
	\left[\hat{R}y_{i}+\hat{P}\right]_{\times} & \underline{\mathbf{0}}_{3}\\
	\mathbf{I}_{3} & \underline{\mathbf{0}}_{3}
	\end{array}\right]^{\top}\overline{{\rm Ad}}_{\hat{\boldsymbol{T}}}\left(\tilde{b}_{U}-W_{U}\right)
	\end{align*}
	Since the last row consists of zeros, the above expression becomes
	\begin{align}
	\dot{e}_{i} & =\dot{\hat{{\rm p}}}_{i}-\left[\begin{array}{c}
	[\hat{R}y_{i}+\hat{P}]_{\times}\\
	\mathbf{I}_{3}
	\end{array}\right]^{\top}\overline{{\rm Ad}}_{\hat{\boldsymbol{T}}}\left(\tilde{b}_{U}-W_{U}\right)\label{eq:SLAM_e_dot_Final}
	\end{align}
	Define the following candidate Lyapunov function $\mathcal{L}=\mathcal{L}(e_{1},e_{2},\ldots,e_{n},\tilde{b}_{U})$
	\begin{equation}
	\mathcal{L}=\sum_{i=1}^{n}\frac{1}{2\alpha_{i}}e_{i}^{\top}e_{i}+\frac{1}{2}\tilde{b}_{U}^{\top}\Gamma^{-1}\tilde{b}_{U}\label{eq:SLAM_Lyap1}
	\end{equation}
	The time derivative of \eqref{eq:SLAM_Lyap1} becomes
	\begin{align}
	\dot{\mathcal{L}}= & \sum_{i=1}^{n}\frac{1}{\alpha_{i}}e_{i}^{\top}\dot{e}_{i}-\tilde{b}_{U}^{\top}\Gamma^{-1}\dot{\hat{b}}_{U}\nonumber \\
	= & -\sum_{i=1}^{n}\frac{1}{\alpha_{i}}e_{i}^{\top}\left[\begin{array}{c}
	[\hat{R}y_{i}+\hat{P}]_{\times}\\
	\mathbf{I}_{3}
	\end{array}\right]^{\top}\overline{{\rm Ad}}_{\hat{\boldsymbol{T}}}\left(\tilde{b}_{U}-W_{U}\right)\nonumber \\
	& +\sum_{i=1}^{n}\frac{1}{\alpha_{i}}e_{i}^{\top}\dot{\hat{{\rm p}}}_{i}-\tilde{b}_{U}^{\top}\Gamma^{-1}\dot{\hat{b}}_{U}\label{eq:SLAM_Lyap1_dot}
	\end{align}
	Substituting $W_{U}$, $\dot{\hat{b}}_{U}$ and $\dot{\hat{{\rm p}}}_{i}$
	with their definitions in \eqref{eq:SLAM_W_f1}, \eqref{eq:SLAM_b_est_dot_f1},
	and \eqref{eq:SLAM_p_est_dot_f1}, respectively, results in the following
	expression: 
	\begin{align}
	\dot{\mathcal{L}}= & -\sum_{i=1}^{n}\frac{k_{1}}{\alpha_{i}}||e_{i}||^{2}-k_{w}\sum_{i=1}^{n}||e_{i}/\alpha_{i}||^{2}\nonumber \\
	& -k_{w}\left\Vert \sum_{i=1}^{n}\left[\hat{R}y_{i}+\hat{P}\right]_{\times}\frac{e_{i}}{\alpha_{i}}\right\Vert ^{2}\label{eq:SLAM_Lyap1_dot_Final}
	\end{align}
	Consistently with the result obtained in \eqref{eq:SLAM_Lyap1_dot_Final}
	the derivative of $\mathcal{L}$ is negative definite with $\dot{\mathcal{L}}$
	being zero at $e_{i}=\underline{\mathbf{0}}_{3}$. Therefore, the
	result in \eqref{eq:SLAM_Lyap1_dot_Final} ensures that $e_{i}$ converges
	exponentially to the set $\mathcal{S}$ defined in \eqref{eq:SLAM_Set1}.
	On the basis of Barbalat Lemma, $\dot{\mathcal{L}}$ is negative,
	continuous and converges to zero. Thus, $\tilde{\boldsymbol{T}}$
	and $\tilde{b}_{U}$ remains bounded as well as $\ddot{e}_{i}$. Moreover,
	according to \eqref{eq:SLAM_e_tilde}, $e_{i}\rightarrow\underline{\mathbf{0}}_{3}$
	implies that $\tilde{{\rm p}}_{i}-\tilde{P}\rightarrow\underline{\mathbf{0}}_{3}$
	which in turn, based on \eqref{eq:SLAM_e}, leads to $\overline{\hat{{\rm p}}}_{i}-\tilde{\boldsymbol{T}}\,\overline{{\rm p}}_{i}\rightarrow\underline{\mathbf{0}}_{3}$.
	Therefore, $\tilde{\boldsymbol{T}}$ is bounded, while $\tilde{R}\rightarrow R_{c}$
	and $\tilde{P}\rightarrow P_{c}$ as $t\rightarrow\infty$. This completes
	the proof.\end{proof}

Let $\Delta t$ denote a small sample time. Algorithm \ref{alg:Alg_Disc-1}
presents the complete steps of implementation of the continuous filter
in \eqref{eq:SLAM_T_est_dot_f1}-\eqref{eq:SLAM_p_est_dot_f1} in
discrete form. $\exp$ in Algorithm \ref{alg:Alg_Disc-1} denotes
exponential of a matrix which is defined in MATLAB as ``expm''.

\begin{algorithm}
	\caption{\label{alg:Alg_Disc-1}Discrete nonlinear filter for SLAM without
		IMU described in Subsection \ref{subsec:Det_without_IMU}}
	
	\textbf{Initialization}:
	\begin{enumerate}
		\item[{\footnotesize{}1:}] Set $\hat{R}[0]\in\mathbb{SO}\left(3\right)$ and $\hat{P}[0]\in\mathbb{R}^{3}$.
		Instead, construct $\hat{R}[0]\in\mathbb{SO}\left(3\right)$ using
		one method of attitude determination, visit \cite{hashim2020AtiitudeSurvey}\vspace{1mm}
		\item[{\footnotesize{}2:}] Set ${\rm \hat{p}}_{i}[0]\in\mathbb{R}^{3}$ for all $i=1,2,\ldots,n$\vspace{1mm}
		\item[{\footnotesize{}3:}] Set $\hat{b}_{U}[0]=0_{6\times1}$ \vspace{1mm}
		\item[{\footnotesize{}4:}] Select $k_{w}$, $k_{1}$, $\Gamma$, and $\alpha_{i}$ as positive
		constants, and the sample $k=0$
	\end{enumerate}
	\textbf{while}
	\begin{enumerate}
		\item[{\footnotesize{}5:}] {\small{}$\overline{{\rm Ad}}_{\hat{\boldsymbol{T}}^{-1}}=\left[\begin{array}{cc}
			\hat{R}[k]^{\top} & 0_{3\times3}\\
			-\hat{R}[k]^{\top}\left[\hat{P}[k]\right]_{\times} & \hat{R}[k]^{\top}
			\end{array}\right]$ }and\\
		{\small{}$\overline{{\rm Ad}}_{\hat{\boldsymbol{T}}}^{\top}=\left[\begin{array}{cc}
			\hat{R}[k]^{\top} & -\hat{R}[k]^{\top}\left[\hat{P}[k]\right]_{\times}\\
			0_{3\times3} & \hat{R}[k]^{\top}
			\end{array}\right]$ }\vspace{1mm}
		\item[{\footnotesize{}6:}] \textbf{for} $i=1:n$\vspace{1mm}
		\item[{\footnotesize{}7:}] \hspace{0.5cm}$e_{i}[k]=\hat{{\rm p}}_{i}[k]-\hat{R}[k]y_{i}[k]-\hat{P}[k]$
		as in \eqref{eq:SLAM_e_Final}\vspace{1mm}
		\item[{\footnotesize{}8:}] \textbf{end for}
		\item[] \textcolor{blue}{/{*} Filter design \& update step {*}/}
		\item[{\footnotesize{}9:}] {\small{}$W_{U}[k]=-\sum_{i=1}^{n}k_{w}\overline{{\rm Ad}}_{\hat{\boldsymbol{T}}^{-1}}\left[\begin{array}{c}
			\left[\hat{R}[k]y_{i}[k]+\hat{P}[k]\right]_{\times}\\
			\mathbf{I}_{3}
			\end{array}\right]e_{i}[k]$}\vspace{1mm}
		\item[{\footnotesize{}10:}] $\hat{\boldsymbol{T}}[k+1]=\hat{\boldsymbol{T}}[k]\exp\left([U_{m}[k]-\hat{b}_{U}[k]-W_{U}[k]]_{\wedge}\Delta t\right)$\vspace{1mm}
		\item[{\footnotesize{}11:}] $\hat{b}_{U}[k+1]=\hat{b}_{U}[k]$\\
		$-\sum_{i=1}^{n}\frac{\Gamma\Delta t}{\alpha_{i}}\overline{{\rm Ad}}_{\hat{\boldsymbol{T}}}^{\top}\left[\begin{array}{c}
		\left[\hat{R}[k]y_{i}[k]+\hat{P}[k]\right]_{\times}\\
		\mathbf{I}_{3}
		\end{array}\right]e_{i}[k]$\vspace{1mm}
		\item[{\footnotesize{}12:}] \textbf{for} $i=1:n$\vspace{1mm}
		\item[{\footnotesize{}13:}] \hspace{0.5cm}${\rm \hat{p}}_{i}[k+1]={\rm \hat{p}}_{i}[k]-\Delta tk_{1}e_{i}[k]$\vspace{1mm}
		\item[{\footnotesize{}14:}] \textbf{end for}\vspace{1mm}
		\item[{\footnotesize{}15:}] $k=k+1$
	\end{enumerate}
	\textbf{end while}
\end{algorithm}

\subsection{Nonlinear Filter Design with IMU\label{subsec:Det_with_IMU}}

The nonlinear filter design presented in Subsection \ref{subsec:Det_without_IMU}
allows $\overset{\circ}{e}_{i}\rightarrow\underline{\mathbf{0}}_{4}$
causing $\tilde{{\rm p}}_{i}-\tilde{P}\rightarrow\underline{\mathbf{0}}_{3}$
exponentially. However, $\tilde{R}\rightarrow R_{c}$ and $\tilde{P}\rightarrow P_{c}$
as $t\rightarrow\infty$ such that $R_{c}\in\mathbb{SO}\left(3\right)$
and $P_{c}\in\mathbb{R}^{3}$ are constants. Recall that $\tilde{P}=\hat{P}-\tilde{R}P$
and $\tilde{{\rm p}}_{i}=\hat{{\rm p}}_{i}-\tilde{R}{\rm p}_{i}$.
Accordingly, if the initial pose of the rigid-body ($R\left(0\right)$
and $P\left(0\right)$) is not accurately known, despite $\overset{\circ}{e}_{i}\rightarrow\underline{\mathbf{0}}_{4}$
exponentially, the error between the following pairs of values will
be very significant: $\hat{R}\left(\infty\right)$ and $R\left(\infty\right)$,
$\hat{P}\left(\infty\right)$ and $P\left(\infty\right)$, and $\hat{{\rm p}}_{i}\left(\infty\right)$
and ${\rm p}_{i}\left(\infty\right)$. As such, the estimates of pose
and feature positions will be highly inaccurate. This is the case
in previously proposed solutions, for instance \cite{zlotnik2018SLAM,hashim2020SLAMLetter}.
\begin{rem}
	\label{rem:SLAM-Observability}SLAM problem is not observable \cite{lee2006SLAM_observability}.
	Let $R_{c}\in\mathbb{SO}\left(3\right)$ and $P_{c}\in\mathbb{R}^{3}$
	be constants. The best achievable result is $\tilde{R}\rightarrow R_{c}$,
	$\tilde{P}\rightarrow P_{c}$, and $\hat{{\rm p}}_{i}\rightarrow\hat{P}+\tilde{R}{\rm p}_{i}-\tilde{R}P$
	as $t\rightarrow\infty$.
\end{rem}
Motivated by the above discussion, this section aims to propose a
nonlinear SLAM filter design that demonstrates reasonable performance
irrespective of the accuracy of the initial pose and feature locations.
The proposed design makes use of the available velocity, IMU, and
feature measurements. Recall the body-frame measurements in \eqref{eq:SLAM_Vect_R}
and their normalization in \eqref{eq:SLAM_Vector_norm}. Let
\begin{equation}
M=M^{\top}=\sum_{j=1}^{n_{{\rm R}}}s_{j}\upsilon_{j}^{r}\left(\upsilon_{j}^{r}\right)^{\top},\hspace{1em}\forall j=1,2,\ldots,n_{{\rm R}}\label{eq:SLAM_M}
\end{equation}
where $s_{j}\geq0$ stands for a constant gain and represents the
confidence level of the $j$th sensor measurements. According to \eqref{eq:SLAM_M},
$M$ is symmetric. In consistence with Remark \ref{rem:R_Marix},
it is assumed that there are at least two body-frame measurements
and their inertial-frame observations are available as well as non-collinear.
Thereby, ${\rm rank}(M)=3$. Let the eigenvalues of $M$ be $\lambda(M)=\{\lambda_{1},\lambda_{2},\lambda_{3}\}$.
Hence, each eigenvalue is positive. Define $\breve{\mathbf{M}}={\rm Tr}\{M\}\mathbf{I}_{3}-M$,
provided that ${\rm rank}(M)=3$. Thus, ${\rm rank}(\breve{\mathbf{M}})=3$
and it can be concluded that (\cite{bullo2004geometric} page. 553): 
\begin{enumerate}
	\item $\breve{\mathbf{M}}$ is positive-definite.
	\item $\breve{\mathbf{M}}$ has the following eigenvalues $\lambda(\breve{\mathbf{M}})=\{\lambda_{3}+\lambda_{2},\lambda_{3}+\lambda_{1},\lambda_{2}+\lambda_{1}\}$
	where the minimum eigenvalue (singular value) $\underline{\lambda}(\breve{\mathbf{M}})>0$. 
\end{enumerate}
The rest of this subsection assumes that ${\rm rank}\left(M\right)=3$.
Also, for $j=1,2,\ldots,n_{{\rm R}}$, $s_{j}$ is selected such that
$\sum_{j=1}^{n_{{\rm R}}}s_{j}=3$. This means that ${\rm Tr}\left\{ M\right\} =3$.
The following Lemma will prove useful in the reminder of this subsection.
\begin{lem}
	\label{Lemm:SLAM_Lemma1}Let $\tilde{R}\in\mathbb{SO}\left(3\right)$,
	$M=M^{\top}\in\mathbb{R}^{3\times3}$ with ${\rm rank}(M)=3$ and
	${\rm Tr}\{M\}=3$. Let $\breve{\mathbf{M}}={\rm Tr}\{M\}\mathbf{I}_{3}-M$
	and $\underline{\lambda}=\underline{\lambda}(\breve{\mathbf{M}})$
	denote the minimum singular value of $\breve{\mathbf{M}}$. Then,
	one has
	\begin{align}
	||\tilde{R}M||_{{\rm I}} & \leq\frac{2}{\underline{\lambda}}\frac{||\mathbf{vex}\left(\boldsymbol{\mathcal{P}}_{a}(\tilde{R}M)\right)||^{2}}{1+{\rm Tr}\{\tilde{R}MM^{-1}\}}\label{eq:SLAM_lemm1_2}
	\end{align}
	\textbf{Proof. See \cite{hashim2019SO3Wiley}}%
	\textbf{.} 
\end{lem}
\begin{defn}
	\label{def:Unstable-set}Define $\mathcal{U}_{s}$ as a subset of
	$\mathbb{SO}\left(3\right)$ which is a non-attractive and forward
	invariant unstable set such that
	\begin{equation}
	\mathcal{U}_{s}=\left\{ \left.\tilde{R}\left(0\right)\in\mathbb{SO}\left(3\right)\right|{\rm Tr}\{\tilde{R}\left(0\right)\}=-1\right\} \label{eq:SO3_PPF_STCH_SET}
	\end{equation}
	with $\tilde{R}\left(0\right)={\rm diag}(1,-1,-1)$, $\tilde{R}\left(0\right)={\rm diag}(-1,1,-1)$,
	or $\tilde{R}\left(0\right)={\rm diag}(-1,-1,1)$ representing the
	only three possible scenarios for $\tilde{R}\left(0\right)\in\mathcal{U}_{s}$.
\end{defn}
The objective of this work is to propose a filter design that relies
on a set of measurements. Therefore, it is important to introduce
the following variables with respect to vector measurements. Recall
\eqref{eq:SLAM_Vect_R} and \eqref{eq:SLAM_Vector_norm}. Since the
true normalized value of the $j$th body-frame vector is equivalent
to $\upsilon_{j}^{a}=R^{\top}\upsilon_{j}^{r}$, define
\begin{equation}
\hat{\upsilon}_{j}^{a}=\hat{R}^{\top}\upsilon_{j}^{r},\hspace{1em}\forall j=1,2,\ldots,n_{{\rm R}}\label{eq:SLAM_vect_R_estimate}
\end{equation}
Let the error in pose be similar to \eqref{eq:SLAM_T_error} such
that $\tilde{R}=\hat{R}R^{\top}$. From the identities in \eqref{eq:SLAM_Identity1}
and \eqref{eq:SLAM_Identity2}, one obtains
\begin{align*}
\left[\hat{R}\sum_{j=1}^{n_{{\rm R}}}\frac{s_{j}}{2}\hat{\upsilon}_{j}^{a}\times\upsilon_{j}^{a}\right]_{\times} & =\hat{R}\sum_{j=1}^{n_{{\rm R}}}\frac{s_{j}}{2}\left(\upsilon_{j}^{a}\left(\hat{\upsilon}_{j}^{a}\right)^{\top}-\hat{\upsilon}_{j}^{a}\left(\upsilon_{j}^{a}\right)^{\top}\right)\hat{R}^{\top}\\
& =\frac{1}{2}\hat{R}R^{\top}M-\frac{1}{2}MR\hat{R}^{\top}\\
& =\boldsymbol{\mathcal{P}}_{a}(\tilde{R}M)
\end{align*}
This implies that $\mathbf{vex}(\boldsymbol{\mathcal{P}}_{a}(\tilde{R}M))$
can be expressed with respect to vector measurements as
\begin{equation}
\boldsymbol{\Upsilon}(\tilde{R}M)=\mathbf{vex}(\boldsymbol{\mathcal{P}}_{a}(\tilde{R}M))=\hat{R}\sum_{j=1}^{n_{{\rm R}}}\left(\frac{s_{j}}{2}\hat{\upsilon}_{j}^{a}\times\upsilon_{j}^{a}\right)\label{eq:SLAM_VEX_VM}
\end{equation}
Hence, $\tilde{R}M$ may be expressed in terms of vector measurements
as
\begin{equation}
\tilde{R}M=\hat{R}\sum_{j=1}^{n_{{\rm R}}}\left(s_{j}\upsilon_{j}^{a}\left(\upsilon_{j}^{r}\right)^{\top}\right)\label{eq:SLAM_RM_VM}
\end{equation}
Due to the fact that ${\rm Tr}\left\{ M\right\} =3$ and in view of
the normalized Euclidean distance definition in \eqref{eq:SLAM_Ecul_Dist},
one finds
\begin{align}
||\tilde{R}M||_{{\rm I}} & =\frac{1}{4}{\rm Tr}\{(\mathbf{I}_{3}-\tilde{R})M\}\nonumber \\
& =\frac{1}{4}{\rm Tr}\left\{ \mathbf{I}_{3}-\hat{R}\sum_{j=1}^{n_{{\rm R}}}\left(s_{j}\upsilon_{j}^{a}\left(\upsilon_{j}^{r}\right)^{\top}\right)\right\} \nonumber \\
& =\frac{1}{4}\sum_{j=1}^{n_{{\rm R}}}\left(1-s_{j}\left(\hat{\upsilon}_{j}^{a}\right)^{\top}\upsilon_{j}^{a}\right)\label{eq:SLAM_RI_VM}
\end{align}
According to \eqref{eq:SLAM_Ecul_Dist}, one obtains 
\begin{align}
1-||\tilde{R}||_{{\rm I}} & =1-\frac{1}{4}{\rm Tr}\left\{ \mathbf{I}_{3}-\tilde{R}\right\} =1-\frac{3}{4}+\frac{1}{4}{\rm Tr}\{\tilde{R}\}\nonumber \\
& =\frac{1}{4}\left(1+{\rm Tr}\{\tilde{R}\}\right)\label{eq:SLAM_Property}
\end{align}
From \eqref{eq:SLAM_Property}, it becomes apparent that
\begin{align}
1-||\tilde{R}||_{{\rm I}} & =\frac{1}{4}\left(1+{\rm Tr}\{\tilde{R}MM^{-1}\}\right)\label{eq:SLAM_property2}
\end{align}
From \eqref{eq:SLAM_property2} and \eqref{eq:SLAM_RM_VM}, one has
\begin{align}
& \pi(\tilde{R},M)={\rm Tr}\{\tilde{R}MM^{-1}\}\nonumber \\
& \hspace{0.3em}={\rm Tr}\left\{ \left(\sum_{j=1}^{n_{{\rm R}}}s_{j}\upsilon_{j}^{a}\left(\upsilon_{j}^{r}\right)^{\top}\right)\left(\sum_{j=1}^{n_{{\rm R}}}s_{j}\hat{\upsilon}_{j}^{a}\left(\upsilon_{j}^{r}\right)^{\top}\right)^{-1}\right\} \label{eq:SLAM_Gamma_VM}
\end{align}
Consider the following nonlinear filter evolved directly on $\mathbb{SLAM}_{n}\left(3\right)$

\begin{align}
\dot{\hat{\boldsymbol{T}}} & =\hat{\boldsymbol{T}}\left[U_{m}-\hat{b}_{U}-W_{U}\right]_{\wedge}\label{eq:SLAM_T_est_dot_f2}\\
\tau_{R} & =\underline{\lambda}(\breve{\mathbf{M}})\times(1+\pi(\tilde{R},M))\label{eq:SLAM_TauR}\\
W_{U} & =\sum_{i=1}^{n}\frac{1}{\alpha_{i}}\left[\begin{array}{cc}
\frac{k_{w}\alpha_{i}}{\tau_{R}}\hat{R}^{\top} & 0_{3\times3}\\
0_{3\times3} & -k_{2}\hat{R}^{\top}
\end{array}\right]\left[\begin{array}{c}
\boldsymbol{\Upsilon}(\tilde{R}M)\\
e_{i}
\end{array}\right]\label{eq:SLAM_W_f2}\\
\dot{\hat{b}}_{U} & =\sum_{i=1}^{n}\frac{\Gamma}{\alpha_{i}}\left[\begin{array}{cc}
\frac{\alpha_{i}}{2}\hat{R}^{\top} & -\left[y_{i}\right]_{\times}\hat{R}^{\top}\\
0_{3\times3} & -\hat{R}^{\top}
\end{array}\right]\left[\begin{array}{c}
\boldsymbol{\Upsilon}(\tilde{R}M)\\
e_{i}
\end{array}\right]\label{eq:SLAM_b_est_dot_f2}\\
\dot{{\rm \hat{p}}}_{i} & =-k_{1}e_{i}+\hat{R}\left[y_{i}\right]_{\times}W_{\Omega},\hspace{1em}i=1,2,\ldots,n\label{eq:SLAM_p_est_dot_f2}
\end{align}
where $W_{U}=\left[W_{\Omega}^{\top},W_{V}^{\top}\right]^{\top}\in\mathbb{R}^{6}$
is a correction factor and $\hat{b}_{U}=\left[\hat{b}_{\Omega}^{\top},\hat{b}_{V}^{\top}\right]^{\top}\in\mathbb{R}^{6}$
is the estimate of $b_{U}$. $k_{w}$, $k_{1}$, $k_{2}$, $\Gamma=\left[\begin{array}{cc}
\Gamma_{1} & 0_{3\times3}\\
0_{3\times3} & \Gamma_{2}
\end{array}\right]$, and $\alpha_{i}$ are positive constants. $M$ is defined in \eqref{eq:SLAM_M},
$\pi(\tilde{R},M)$ and $\boldsymbol{\Upsilon}(\tilde{R}M)$ are found
in \eqref{eq:SLAM_Gamma_VM} and \eqref{eq:SLAM_VEX_VM}, respectively,
while $e_{i}$ is defined in \eqref{eq:SLAM_e_Final} for all $i=1,2,\cdots,n$.
\begin{thm}
	Consider the SLAM dynamics $\dot{X}=(\dot{\boldsymbol{T}},\dot{\overline{{\rm p}}})$
	in \eqref{eq:SLAM_True_dot} with measurements obtained from features
	(output $\overline{y}_{i}=\boldsymbol{T}^{-1}\overline{{\rm p}}_{i}$)
	for all $i=1,2,\ldots,n$, inertial measurement units $\upsilon_{j}^{a}=R^{\top}\upsilon_{j}^{r}$
	for all $j=1,2,\ldots,n_{{\rm R}}$ and velocity measurements ($U_{m}=U+b_{U}$).
	Let Assumption \ref{Assumption:Feature} hold along with the discussion
	in Remark \ref{rem:R_Marix} ($n_{{\rm R}}\geq2$). Assume the filter
	design to be as in \eqref{eq:SLAM_T_est_dot_f2}, \eqref{eq:SLAM_TauR},
	\eqref{eq:SLAM_W_f2}, \eqref{eq:SLAM_b_est_dot_f2}, and \eqref{eq:SLAM_p_est_dot_f2}
	combined with the measurements $U_{m}$, $\upsilon_{j}^{a}$ and $\overline{y}_{i}$.
	Consider the design parameters $k_{w}$, $k_{1}$, $k_{2}$, $\Gamma$,
	and $\alpha_{i}$ to be positive constants for all $i=1,2,\ldots,n$,
	and $j=1,2,\ldots n_{{\rm R}}$. Define the following set:
	\begin{align}
	\mathcal{S}= & \left\{ (\tilde{R},e_{1},e_{2},\ldots,e_{n})\in\mathbb{SO}\left(3\right)\times\mathbb{R}^{3}\times\mathbb{R}^{3}\times\cdots\times\mathbb{R}^{3}\right|\nonumber \\
	& \hspace{7em}\left.\tilde{R}=\mathbf{I}_{3},e_{i}=\underline{\mathbf{0}}_{3}\forall i=1,2,\ldots n\right\} \label{eq:SLAM_Set2}
	\end{align}
	Then, 1) the error $(\tilde{R},e_{1},\ldots,e_{n})$ converges exponentially
	to $\mathcal{S}$ from almost any initial condition ($\tilde{R}\left(0\right)\notin\mathcal{U}_{s}$),
	2) $\tilde{b}_{U}$ converges asymptotically to the origin, and 3)
	the trajectory of $\tilde{P}$ remains bounded and there exists a
	constant vector $P_{c}\in\mathbb{R}^{3}$ with $\lim_{t\rightarrow\infty}\tilde{P}=P_{c}$.
\end{thm}
\begin{proof} Since $\dot{\hat{\boldsymbol{T}}}$ in \eqref{eq:SLAM_T_est_dot_f2}
	is similar to \eqref{eq:SLAM_T_est_dot_f1}, the pose error dynamics
	become similar to \eqref{eq:SLAM_T_error_dot} such that $\dot{\tilde{\boldsymbol{T}}}={\rm Ad}_{\hat{\boldsymbol{T}}}\left(\left[\tilde{b}_{U}-W_{U}\right]_{\wedge}\right)\tilde{\boldsymbol{T}}$.
	Thus, the attitude error dynamics are
	\begin{align}
	\dot{\tilde{R}} & =\dot{\hat{R}}R^{\top}+\hat{R}\dot{R}^{\top}=\hat{R}\left[\tilde{b}_{\Omega}-\hat{R}^{\top}W_{\Omega}\right]_{\times}R^{\top}\nonumber \\
	& =\left[\hat{R}\tilde{b}_{\Omega}-W_{\Omega}\right]_{\times}\tilde{R}\label{eq:SLAM_Attit_Error_dot}
	\end{align}
	Recall the normalized Euclidean distance definition in \eqref{eq:SLAM_Ecul_Dist}
	such that $||\tilde{R}M||_{{\rm I}}=\frac{1}{4}{\rm Tr}\left\{ (\mathbf{I}_{3}-\tilde{R})M\right\} $.
	Thereby, in view of \eqref{eq:SLAM_Identity6}, one has
	\begin{align}
	\frac{d}{dt}||\tilde{R}M||_{{\rm I}} & =-\frac{1}{4}{\rm Tr}\left\{ \left[\hat{R}(\tilde{b}_{\Omega}-W_{\Omega})\right]_{\times}\tilde{R}M\right\} \nonumber \\
	& =-\frac{1}{4}{\rm Tr}\left\{ \tilde{R}M\boldsymbol{\mathcal{P}}_{a}\left(\left[\hat{R}(\tilde{b}_{\Omega}-W_{\Omega})\right]_{\times}\right)\right\} \nonumber \\
	& =\frac{1}{2}\mathbf{vex}\left(\boldsymbol{\mathcal{P}}_{a}(\tilde{R}M)\right)^{\top}\hat{R}(\tilde{b}_{\Omega}-W_{\Omega})\label{eq:SLAM_RI_VM_dot}
	\end{align}
	Note that $\dot{M}=0_{3\times3}$ by its definition in \eqref{eq:SLAM_M}.
	Recalling the expression in \eqref{eq:SLAM_expression1}, one finds
	\begin{align}
	& {\rm Ad}_{\hat{\boldsymbol{T}}}\left(\left[\tilde{b}_{U}-W_{U}\right]_{\wedge}\right)\tilde{\boldsymbol{T}}\,\overline{{\rm p}}_{i}\nonumber \\
	& \hspace{3em}=\left[\begin{array}{cc}
	\left[\hat{R}y_{i}+\hat{P}\right]_{\times} & \underline{\mathbf{0}}_{3}\\
	\mathbf{I}_{3} & \underline{\mathbf{0}}_{3}
	\end{array}\right]^{\top}\overline{{\rm Ad}}_{\hat{\boldsymbol{T}}}\left(\tilde{b}_{U}-W_{U}\right)\nonumber \\
	& \hspace{3em}=\left[\begin{array}{cc}
	-\hat{R}\left[y_{i}\right]_{\times} & \hat{R}\\
	\underline{\mathbf{0}}_{3}^{\top} & \underline{\mathbf{0}}_{3}^{\top}
	\end{array}\right]\left(\tilde{b}_{U}-W_{U}\right)\label{eq:SLAM_expression2}
	\end{align}
	Thus, analogously to \eqref{eq:SLAM_e_dot_Final} and in view of \eqref{eq:SLAM_expression2},
	the error dynamics of $\overset{\circ}{e}_{i}$ can be expressed as
	\begin{align*}
	\overset{\circ}{\dot{e}}_{i} & =\overset{\circ}{\dot{\hat{{\rm p}}}}_{i}-\left[\begin{array}{cc}
	-\hat{R}\left[y_{i}\right]_{\times} & \hat{R}\\
	\underline{\mathbf{0}}_{3}^{\top} & \underline{\mathbf{0}}_{3}^{\top}
	\end{array}\right]\left(\tilde{b}_{U}-W_{U}\right)
	\end{align*}
	which means
	\begin{align}
	\dot{e}_{i} & =\dot{\hat{{\rm p}}}_{i}-\left[\begin{array}{cc}
	-\hat{R}\left[y_{i}\right]_{\times} & \hat{R}\end{array}\right]\left(\tilde{b}_{U}-W_{U}\right)\label{eq:SLAM_e_dot_Final2}
	\end{align}
	Define the following candidate Lyapunov function $\mathcal{L}=\mathcal{L}(e_{1},e_{2},\ldots,e_{n},||\tilde{R}M||_{{\rm I}},\tilde{b}_{U})$
	\begin{equation}
	\mathcal{L}=\sum_{i=1}^{n}\frac{1}{2\alpha_{i}}e_{i}^{\top}e_{i}+||\tilde{R}M||_{{\rm I}}+\frac{1}{2}\tilde{b}_{U}^{\top}\Gamma^{-1}\tilde{b}_{U}\label{eq:SLAM_Lyap2}
	\end{equation}
	From \eqref{eq:SLAM_RI_VM_dot} and \eqref{eq:SLAM_e_dot_Final2},
	the time derivative of \eqref{eq:SLAM_Lyap2} becomes
	\begin{align}
	\dot{\mathcal{L}}= & \sum_{i=1}^{n}\frac{1}{\alpha_{i}}e_{i}^{\top}\dot{e}_{i}+||\dot{\tilde{R}}M||_{{\rm I}}-\tilde{b}_{U}^{\top}\Gamma^{-1}\dot{\hat{b}}_{U}\nonumber \\
	= & \sum_{i=1}^{n}\frac{1}{\alpha_{i}}e_{i}^{\top}\dot{\hat{{\rm p}}}_{i}-\sum_{i=1}^{n}\frac{1}{\alpha_{i}}e_{i}^{\top}\left[\begin{array}{cc}
	-\hat{R}\left[y_{i}\right]_{\times} & \hat{R}\end{array}\right](\tilde{b}_{U}-W_{U})\nonumber \\
	& +\frac{1}{2}\boldsymbol{\Upsilon}(\tilde{R}M)^{\top}\hat{R}(\tilde{b}_{\Omega}-W_{\Omega})-\tilde{b}_{U}^{\top}\Gamma^{-1}\dot{\hat{b}}_{U}\label{eq:SLAM_Lyap2_dot2}
	\end{align}
	which means
	\begin{align}
	\dot{\mathcal{L}}= & \sum_{i=1}^{n}\frac{1}{\alpha_{i}}\left[\begin{array}{c}
	\boldsymbol{\Upsilon}(\tilde{R}M)\\
	e_{i}
	\end{array}\right]^{\top}\left[\begin{array}{cc}
	\frac{\alpha_{i}}{2}\hat{R} & 0_{3\times3}\\
	\hat{R}\left[y_{i}\right]_{\times} & -\hat{R}
	\end{array}\right](\tilde{b}_{U}-W_{U})\nonumber \\
	& +\sum_{i=1}^{n}\frac{1}{\alpha_{i}}e_{i}^{\top}\dot{\hat{{\rm p}}}_{i}-\tilde{b}_{U}^{\top}\Gamma^{-1}\dot{\hat{b}}_{U}\label{eq:SLAM_Lyap2_dot3}
	\end{align}
	With direct substitution of $W_{U}$, $\dot{\hat{b}}_{U}$, and $\dot{\hat{{\rm p}}}_{i}$
	with their definitions in \eqref{eq:SLAM_W_f2}, \eqref{eq:SLAM_b_est_dot_f2},
	and \eqref{eq:SLAM_p_est_dot_f2}, respectively, one obtains
	\begin{align*}
	\dot{\mathcal{L}}= & -\sum_{i=1}^{n}\frac{k_{1}}{\alpha_{i}}||e_{i}||^{2}-\frac{k_{w}}{2\tau_{R}}\left\Vert \boldsymbol{\Upsilon}(\tilde{R}M)\right\Vert ^{2}-k_{2}\sum_{i=1}^{n}||e_{i}/\alpha_{i}||^{2}
	\end{align*}
	As a result of \eqref{eq:SLAM_lemm1_2} in Lemma \ref{Lemm:SLAM_Lemma1},
	one obtains
	\begin{equation}
	\dot{\mathcal{L}}\leq-\sum_{i=1}^{n}\frac{k_{1}}{\alpha_{i}}||e_{i}||^{2}-\frac{k_{w}}{4}||\tilde{R}M||_{{\rm I}}-k_{2}\sum_{i=1}^{n}||e_{i}/\alpha_{i}||^{2}\label{eq:SLAM_Lyap2_final}
	\end{equation}
	According to the result in \eqref{eq:SLAM_Lyap2_final}, the derivative
	of $\dot{\mathcal{L}}$ is negative definite, while $\dot{\mathcal{L}}$
	equals to zero at $e_{i}=\underline{\mathbf{0}}_{3}$ as well as $||\tilde{R}M||_{{\rm I}}=0$.
	By the definition of the normalized Euclidean distance $||\tilde{R}M||_{{\rm I}}=\frac{1}{4}{\rm Tr}\left\{ (\mathbf{I}_{3}-\tilde{R})M\right\} $,
	$||\tilde{R}M||_{{\rm I}}=0$ if and only if $\tilde{R}=\mathbf{I}_{3}$.
	Thus, the result in \eqref{eq:SLAM_Lyap2_final} ensures that $e_{i}$
	as well as $\tilde{R}$ converge exponentially to the set $\mathcal{S}$
	defined in \eqref{eq:SLAM_Set2} for all $i=1,2,\ldots,n$ and $\tilde{R}\left(0\right)\notin\mathcal{U}_{s}$.
	Based on \eqref{eq:SLAM_lemm1_2} in Lemma \ref{Lemm:SLAM_Lemma1}
	along with the definitions in \eqref{eq:SLAM_RI_VM} and \eqref{eq:SLAM_VEX_VM},
	$||\tilde{R}M||_{{\rm I}}\rightarrow0$ implies that $\boldsymbol{\Upsilon}(\tilde{R}M)\rightarrow0$.
	$\dot{\mathcal{L}}$ is negative, continuous and converges to zero
	signifying that $\mathcal{L}\in\mathcal{L}_{\infty}$ and that a finite
	$\lim_{t\rightarrow\infty}\mathcal{L}$ exists. In view of $\tilde{b}_{U}$
	definition in \eqref{eq:SLAM_b_error} and $\dot{\hat{b}}_{U}$ in
	\eqref{eq:SLAM_b_est_dot_f2}, $\dot{\tilde{b}}_{U}=-\dot{\hat{b}}_{U}$
	implies that $\dot{\tilde{b}}_{U}\rightarrow0$ as $e_{i}\rightarrow0$
	and $\boldsymbol{\Upsilon}(\tilde{R}M)\rightarrow0$. Thereby, $\tilde{b}_{U}$
	is bounded for all $t\geq0$. In view of \eqref{eq:SLAM_W_f2}, $W_{U}\rightarrow0$
	as $e_{i}\rightarrow0$ and $\boldsymbol{\Upsilon}(\tilde{R}M)\rightarrow0$.
	Moreover, from \eqref{eq:SLAM_p_est_dot_f2}, $\dot{{\rm \hat{p}}}_{i}\rightarrow0$
	as $e_{i}\rightarrow0$ and $W_{U}\rightarrow0$. Since $\lim_{t\rightarrow\infty}\dot{e}_{i}=0$
	and considering the above discussion, one has
	\[
	\lim_{t\rightarrow\infty}\dot{e}_{i}=\lim_{t\rightarrow\infty}-\left[\begin{array}{c}
	\left[\hat{R}y_{i}+\hat{P}\right]_{\times}\\
	\mathbf{I}_{3}
	\end{array}\right]^{\top}\overline{{\rm Ad}}_{\hat{\boldsymbol{T}}}\tilde{b}_{U}=0
	\]
	Define
	\[
	N=-\left[\begin{array}{cc}
	-\left[\hat{R}y_{1}+\hat{P}\right]_{\times} & \mathbf{I}_{3}\\
	\vdots & \vdots\\
	-\left[\hat{R}y_{n}+\hat{P}\right]_{\times} & \mathbf{I}_{3}
	\end{array}\right]\overline{{\rm Ad}}_{\hat{\boldsymbol{T}}}\in\mathbb{R}^{3n\times6},\hspace{1em}n\geq3
	\]
	Consistently with Assumption \ref{Assumption:Feature}, number of
	features is $n\geq3$. Accordingly, $N$ is full column rank. It becomes
	apparent that $\lim_{t\rightarrow\infty}N\tilde{b}_{U}=0$ implies
	that $\lim_{t\rightarrow\infty}\tilde{b}_{U}=0$. Hence, from \eqref{eq:SLAM_Lyap2_final},
	$\ddot{\mathcal{L}}$ is bounded. Based on Barbalat Lemma, $\dot{\mathcal{L}}$
	is uniformly continuous. Due to the fact that $\tilde{b}_{U}\rightarrow0$
	and $W_{U}\rightarrow0$ as $t\rightarrow\infty$, $\dot{\tilde{\boldsymbol{T}}}\rightarrow0$
	which leads to $\tilde{\boldsymbol{T}}\rightarrow\boldsymbol{T}_{c}(\mathbf{I}_{3},P_{c})$
	with $\boldsymbol{T}_{c}(\mathbf{I}_{3},P_{c})\in\mathbb{SE}\left(3\right)$
	denoting a constant matrix where $P_{c}\in\mathbb{R}^{3}$ is a constant
	vector. Therefore, it can be concluded that $\lim_{t\rightarrow\infty}\tilde{P}=P_{c}$
	completing the proof.\end{proof}
\begin{rem}
	Selecting $\tau_{R}=1$ in \eqref{eq:SLAM_TauR} will lead to
	\[
	\dot{\mathcal{L}}=-\sum_{i=1}^{n}\frac{k_{1}}{\alpha_{i}}||e_{i}||^{2}-\frac{k_{w}}{2}\left\Vert \boldsymbol{\Upsilon}(\tilde{R}M)\right\Vert ^{2}-k_{2}\sum_{i=1}^{n}||e_{i}/\alpha_{i}||^{2}
	\]
	such that $e_{i}\rightarrow0$ and $\tilde{R}\rightarrow\mathbf{I}_{3}$
	asymptotically for all $i=1,2,\ldots,n$ and $\tilde{R}\left(0\right)\notin\mathcal{U}_{s}$.
\end{rem}
\begin{algorithm}
	\caption{\label{alg:Alg_Disc}Discrete nonlinear filter for SLAM described
		in Subsection \ref{subsec:Det_with_IMU}}
	
	\textbf{Initialization}:
	\begin{enumerate}
		\item[{\footnotesize{}1:}] Set $\hat{R}[0]\in\mathbb{SO}\left(3\right)$ and $\hat{P}[0]\in\mathbb{R}^{3}$.
		Alternatively, construct $\hat{R}[0]\in\mathbb{SO}\left(3\right)$
		using one of the methods of attitude determination, visit \cite{hashim2020AtiitudeSurvey}\vspace{1mm}
		\item[{\footnotesize{}2:}] Set ${\rm \hat{p}}_{i}[0]\in\mathbb{R}^{3}$ for all $i=1,2,\ldots,n$\vspace{1mm}
		\item[{\footnotesize{}3:}] Set $\hat{b}_{U}[0]=0_{6\times1}$ \vspace{1mm}
		\item[{\footnotesize{}4:}] Select $k_{w}$, $k_{1}$, $k_{2}$, $\Gamma$, and $\alpha_{i}$
		as positive constants, and the sample $k=0$
	\end{enumerate}
	\textbf{while }(1)\textbf{ do}
	\begin{enumerate}
		\item[] \textcolor{blue}{/{*} Measurement collection \& Filter setup {*}/}
		\item[{\footnotesize{}5:}] \textbf{for} $j=1:n_{{\rm R}}$\vspace{1mm}
		\item[{\footnotesize{}6:}] \hspace{0.5cm}Measurements and observations as in \eqref{eq:SLAM_Vect_R}\vspace{1mm}
		\item[{\footnotesize{}7:}] \hspace{0.5cm}$\upsilon_{j}^{r}=\frac{r_{j}}{\left\Vert r_{j}\right\Vert },\upsilon_{j}^{a}=\upsilon_{j}^{a}[k]=\frac{a_{j}[k]}{\left\Vert a_{j}[k]\right\Vert }$
		as in \eqref{eq:SLAM_Vector_norm}\vspace{1mm}
		\item[{\footnotesize{}8:}] \hspace{0.5cm}$\hat{\upsilon}_{j}^{a}=\hat{\upsilon}_{j}^{a}[k]=\hat{R}[k]^{\top}\upsilon_{j}^{r}$
		as in \eqref{eq:SLAM_vect_R_estimate}\vspace{1mm}
		\item[{\footnotesize{}9:}] \textbf{end for}\vspace{1mm}
		\item[{\footnotesize{}10:}] $M=\sum_{j=1}^{n_{{\rm R}}}s_{j}\upsilon_{j}^{r}\left(\upsilon_{j}^{r}\right)^{\top}$
		as in \eqref{eq:SLAM_M} with\\
		$\breve{\mathbf{M}}={\rm Tr}\left\{ M\right\} \mathbf{I}_{3}-M$\vspace{1mm}
		\item[{\footnotesize{}11:}] $\boldsymbol{\Upsilon}=\boldsymbol{\Upsilon}[k]=\hat{R}\sum_{j=1}^{n_{{\rm R}}}(\frac{s_{j}}{2}\hat{\upsilon}_{j}^{a}\times\upsilon_{j}^{a})$
		as in \eqref{eq:SLAM_VEX_VM}\vspace{1mm}
		\item[{\footnotesize{}12:}] $\pi={\rm Tr}\left\{ \left(\sum_{j=1}^{n_{{\rm R}}}s_{j}\upsilon_{j}^{a}\left(\upsilon_{j}^{r}\right)^{\top}\right)\left(\sum_{j=1}^{n_{{\rm R}}}s_{j}\hat{\upsilon}_{j}^{a}\left(\upsilon_{j}^{r}\right)^{\top}\right)^{-1}\right\} $
		as in \eqref{eq:SLAM_Gamma_VM} where $\pi=\pi[k]$\vspace{1mm}
		\item[{\footnotesize{}13:}] \textbf{for} $i=1:n$\vspace{1mm}
		\item[{\footnotesize{}14:}] \hspace{0.5cm}$e_{i}[k]=\hat{{\rm p}}_{i}[k]-\hat{R}[k]y_{i}[k]-\hat{P}[k]$
		as in \eqref{eq:SLAM_e_Final}\vspace{1mm}
		\item[{\footnotesize{}15:}] \textbf{end for}\vspace{1mm}
		\item[] \textcolor{blue}{/{*} Filter design \& update step {*}/}
		\item[{\footnotesize{}16:}] $W_{U}[k]=\left[\begin{array}{c}
		W_{\Omega}[k]\\
		W_{V}[k]
		\end{array}\right]=\sum_{i=1}^{n}\frac{1}{\alpha_{i}}\left[\begin{array}{c}
		\frac{k_{w}\alpha_{i}}{\tau_{R}}\hat{R}[k]^{\top}\boldsymbol{\Upsilon}\\
		-k_{2}\hat{R}[k]^{\top}e_{i}[k]
		\end{array}\right]$,\\
		with $\tau_{R}=\underline{\lambda}(\breve{\mathbf{M}})\times(1+\pi)$\vspace{1mm}
		\item[{\footnotesize{}17:}] $\hat{\boldsymbol{T}}[k+1]=\hat{\boldsymbol{T}}[k]\exp([U_{m}[k]-\hat{b}_{U}[k]-W_{U}[k]]_{\wedge}\Delta t)$\vspace{1mm}
		\item[{\footnotesize{}18:}] $\hat{b}_{U}[k+1]=\hat{b}_{U}[k]$\\
		$-\sum_{i=1}^{n}\frac{\Gamma\Delta t}{\alpha_{i}}\left[\begin{array}{cc}
		\frac{\alpha_{i}}{2}\hat{R}[k]^{\top} & -\left[y_{i}[k]\right]_{\times}\hat{R}[k]^{\top}\\
		0_{3\times3} & -\hat{R}[k]^{\top}
		\end{array}\right]\left[\begin{array}{c}
		\boldsymbol{\Upsilon}\\
		e_{i}[k]
		\end{array}\right]$\vspace{1mm}
		\item[{\footnotesize{}19:}] \textbf{for} $i=1:n$\vspace{1mm}
		\item[{\footnotesize{}20:}] \hspace{0.5cm}${\rm \hat{p}}_{i}[k+1]={\rm \hat{p}}_{i}[k]-\Delta t(k_{1}e_{i}[k]+\hat{R}[k]\left[y_{i}[k]\right]_{\times}W_{\Omega}[k])$\vspace{1mm}
		\item[{\footnotesize{}21:}] \textbf{end for}\vspace{1mm}
		\item[{\footnotesize{}22:}] $k=k+1$
	\end{enumerate}
	\textbf{end while}
\end{algorithm}

The continuous form of the filter proposed in \eqref{eq:SLAM_T_est_dot_f2}-\eqref{eq:SLAM_p_est_dot_f2}
can be simplified and summarized in terms of vector measurements as
follows:
\begin{equation}
\begin{cases}
\dot{\hat{R}} & =\hat{R}\left[\Omega_{m}-\hat{b}_{\Omega}-W_{\Omega}\right]_{\times}\\
\dot{\hat{P}} & =\hat{R}(V_{m}-\hat{b}_{V}-W_{V})\\
\tau_{R} & =\underline{\lambda}(\breve{\mathbf{M}})\times(1+\pi(\tilde{R},M))\\
W_{\Omega} & =\frac{k_{w}}{\tau_{R}}\hat{R}^{\top}\boldsymbol{\Upsilon}(\tilde{R}M)\\
W_{V} & =-\sum_{i=1}^{n}\frac{k_{2}}{\alpha_{i}}\hat{R}^{\top}e_{i}\\
\dot{\hat{b}}_{\Omega} & =\frac{\Gamma_{1}}{2}\hat{R}^{\top}\boldsymbol{\Upsilon}(\tilde{R}M)-\sum_{i=1}^{n}\frac{\Gamma_{1}}{\alpha_{i}}[y_{i}]_{\times}\hat{R}^{\top}e_{i}\\
\dot{\hat{b}}_{V} & =-\sum_{i=1}^{n}\frac{\Gamma_{2}}{\alpha_{i}}\hat{R}^{\top}e_{i}\\
\dot{{\rm \hat{p}}}_{i} & =-k_{1}e_{i}+\hat{R}[y_{i}]_{\times}W_{\Omega},\hspace{1em}i=1,2,\ldots,n
\end{cases}\label{eq:SLAM_Filter_Pieces}
\end{equation}
Let $\Delta t$ denote a small sample time. The detailed implementation
steps of the discrete form of the filter proposed in \eqref{eq:SLAM_T_est_dot_f2}-\eqref{eq:SLAM_p_est_dot_f2}
can be found in Algorithm \ref{alg:Alg_Disc}. It should be remarked
that $\exp$ in Algorithm \ref{alg:Alg_Disc} denotes exponential
of a matrix which is defined in MATLAB as ``expm''.

\section{Simulation Results \label{sec:SE3_Simulations}}

In this section, the effectiveness of the proposed nonlinear filter
for SLAM on the Lie group $\mathbb{SLAM}_{n}\left(3\right)$ is put
to the test. Let the angular velocity be $\Omega=[0,0,0.3]^{\top}({\rm rad/sec})$
and the translational velocity be $V=[2.5,0,0.2t]^{\top}({\rm m/sec})$.
Consider the following initial values of the true attitude and position
of the vehicle
\[
R\left(0\right)=\mathbf{I}_{3},\hspace{1em}P\left(0\right)=[0,0,6]^{\top}
\]
Let us place the four features fixed in space relative to the inertial
frame at ${\rm p}_{1}=[10,10,0]^{\top}$, ${\rm p}_{2}=[-10,10,0]^{\top}$,
${\rm p}_{3}=[10,-10,0]^{\top}$, and ${\rm p}_{4}=[-10,-10,0]^{\top}$.
Suppose that unknown bias is corrupting the group velocity vector
$b_{U}=\left[b_{\Omega}^{\top},b_{V}^{\top}\right]^{\top}$ with $b_{\Omega}=[0.2,-0.2,0.2]^{\top}({\rm rad/sec})$
and $b_{V}=[0.04,0.1,-0.02]^{\top}({\rm m/sec})$. In addition, let
us assume that the group velocity vector is corrupted with noise defined
as $n_{U}=\left[n_{\Omega}^{\top},n_{V}^{\top}\right]^{\top}$ where
$n_{\Omega}=\mathcal{N}\left(0,0.2\right)({\rm rad/sec})$ and $n_{V}=\mathcal{N}\left(0,0.2\right)({\rm m/sec})$.
Note that $n_{\Omega}=\mathcal{N}\left(0,0.2\right)$ is a short-hand
notation for a normally distributed random noise vector with zero
mean and a standard deviation of $0.2$. Let two non-collinear inertial-frame
observations be given as $r_{1}=\left[1,-1,1\right]^{\top}$ and $r_{2}=\left[0,0,1\right]^{\top}$,
and define the body-frame measurements as in \eqref{eq:SLAM_Vect_R}.
In accordance with Remark \ref{rem:R_Marix}, let us obtain the third
observation and the associated measurements by means of a cross product
of the two available observations. Let the initial estimates of attitude
and position be set to
\begin{align*}
\hat{R}\left(0\right) & =\left[\begin{array}{ccc}
0.8112 & -0.5660 & 0.1468\\
0.5749 & 0.8179 & -0.0234\\
-0.1068 & 0.1034 & 0.9889
\end{array}\right]\\
\hat{P}\left(0\right) & =[0,0,0]^{\top}
\end{align*}
and suppose that the initial feature position estimates are set to
$\hat{{\rm p}}_{1}\left(0\right)=\hat{{\rm p}}_{2}\left(0\right)=\hat{{\rm p}}_{3}\left(0\right)=\hat{{\rm p}}_{4}\left(0\right)=[0,0,0]^{\top}$.
Consider the design parameters to be $\alpha_{i}=0.1$, $\Gamma_{1}=3\mathbf{I}_{3}$,
$\Gamma_{2}=100\mathbf{I}_{3}$, $k_{w}=5$, $k_{1}=5$, and $k_{2}=20$,
with the initial bias estimate being $\hat{b}_{U}\left(0\right)=\underline{\mathbf{0}}_{6}$
for all $i=1,2,3,4$.

The illustration of the true angular and translational velocities
plotted against their measurements can be seen in Fig. \ref{fig:SLAM_Vel}
(two of the three components). Fig. \ref{fig:SLAM_3d} demonstrates
evolution of the trajectories estimated by the nonlinear filter for
SLAM presented in Subsection IV-B in its continuous form. Although
the trajectory of the vehicle was initialized with a large error,
Fig. \ref{fig:SLAM_3d} shows how it was smoothly regulated to the
true trajectory ultimately reaching the desired destination. Likewise,
feature estimates initialized at the origin gradually diverged to
their true respective positions. 

\begin{figure}[h!]
	\centering{}\includegraphics[scale=0.31]{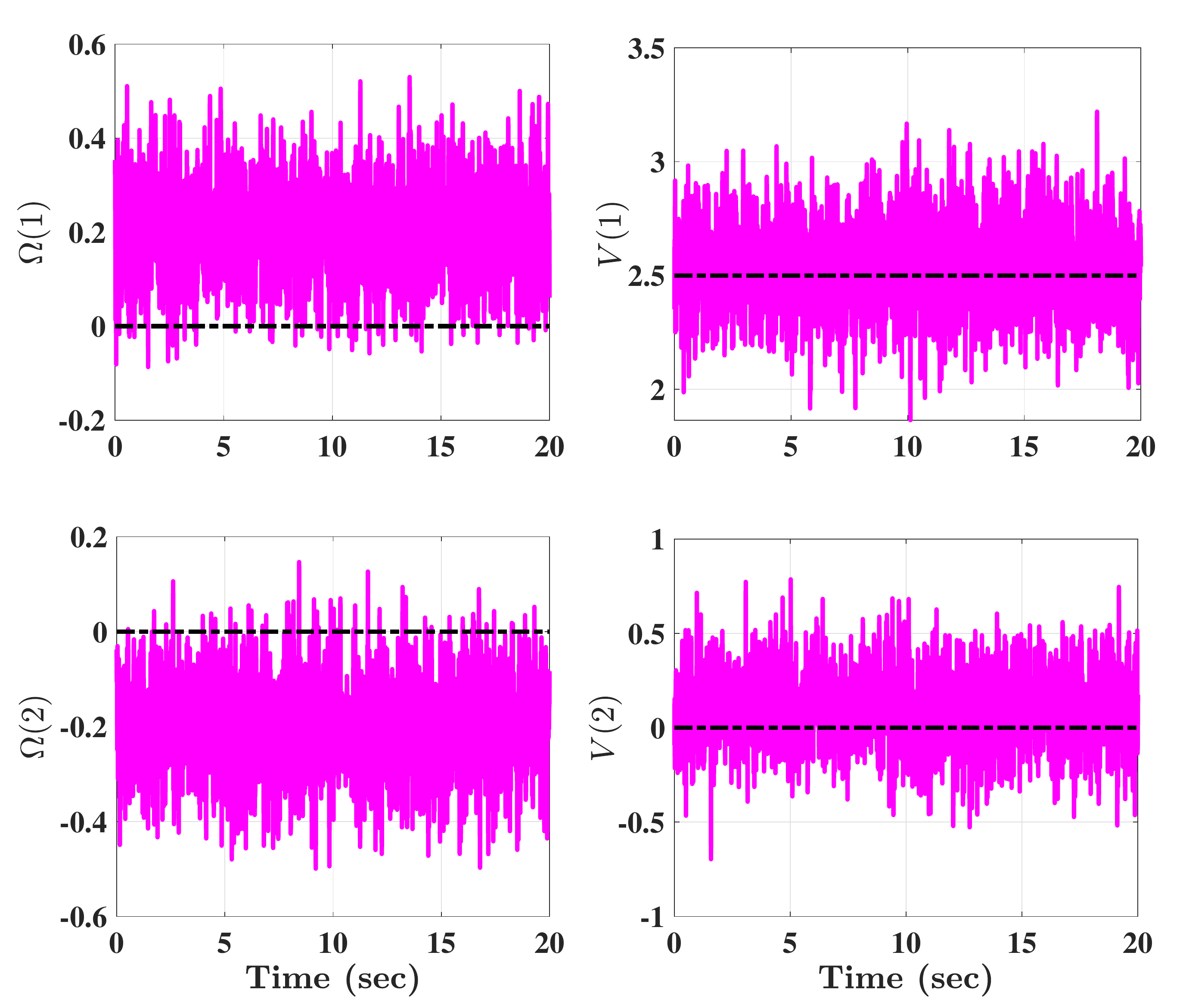}\caption{True angular and translational velocities plotted in black center-line
		vs measurements of angular and translational velocities plotted in
		magenta solid-line.}
	\label{fig:SLAM_Vel}
\end{figure}

\begin{figure}[h!]
	\centering{}\includegraphics[scale=0.31]{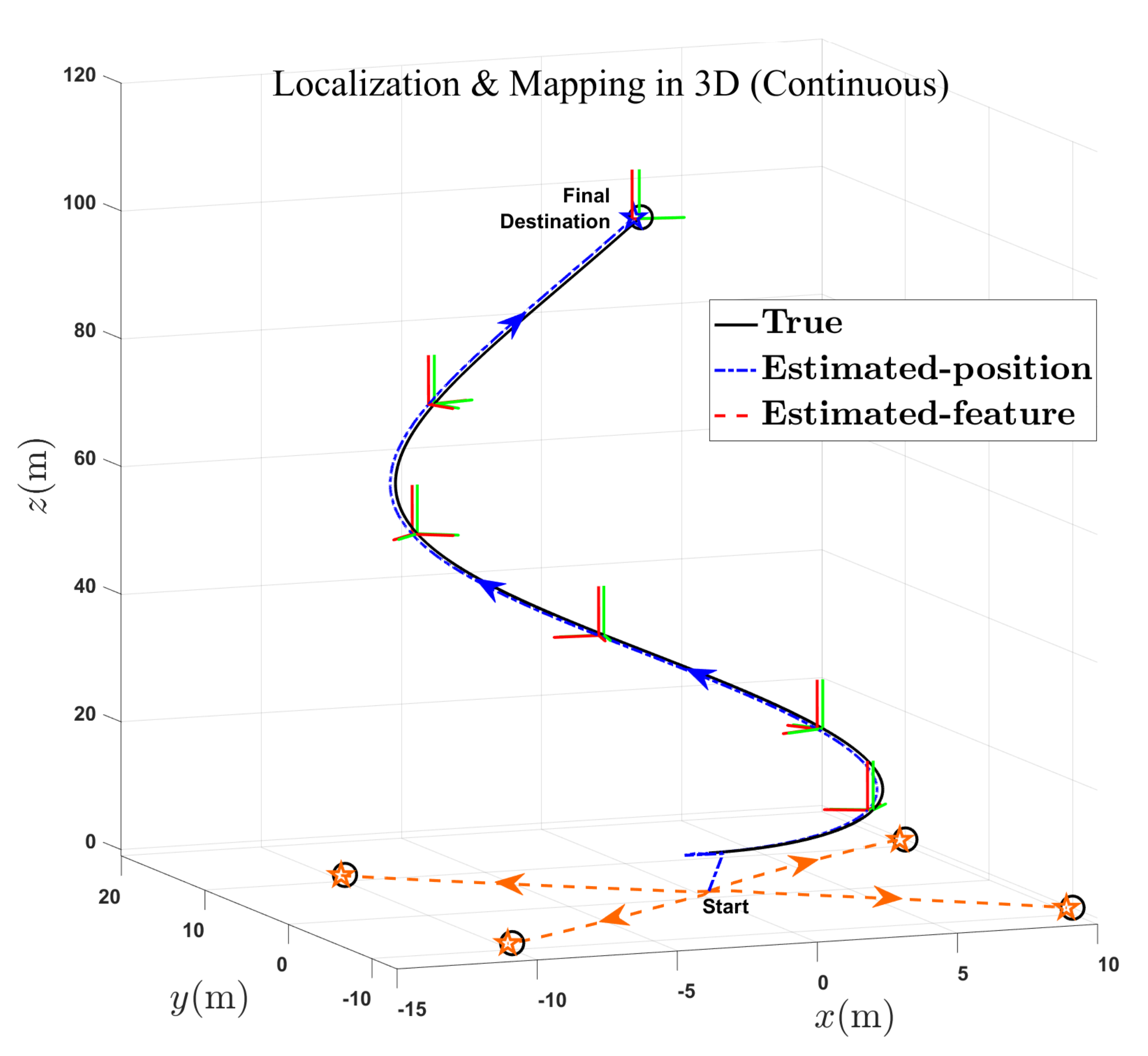}\caption{Output trajectories of the proposed nonlinear continuous filter for
		SLAM described in Subsection \ref{subsec:Det_with_IMU} using velocity,
		feature, and IMU measurements are plotted against the true vehicle
		and feature positions in 3D space (continuous time). The true vehicle
		trajectory is depicted in black solid-line with its final destination
		marked as a black circle. The true orientation of the vehicle is depicted
		as a green solid-line. Additionally, the true fixed features are marked
		as black circle at ${\rm p}_{1}$, ${\rm p}_{2}$, ${\rm p}_{3}$,
		and ${\rm p}_{4}$. The estimation of the travel trajectory is shown
		as a blue center-line starting from (0,0,0) and ending at its final
		destination shown as a blue star \textcolor{blue}{$\star$}. The vehicle
		orientation estimation is shown as a red solid-line. The estimation
		process of the feature positions is shown in orange dash-lines which
		originate at (0,0,0) and end at their final destinations marked with
		orange stars \textcolor{orange}{$\star$}.}
	\label{fig:SLAM_3d}
\end{figure}

\begin{figure}[h!]
	\centering{}\includegraphics[scale=0.3]{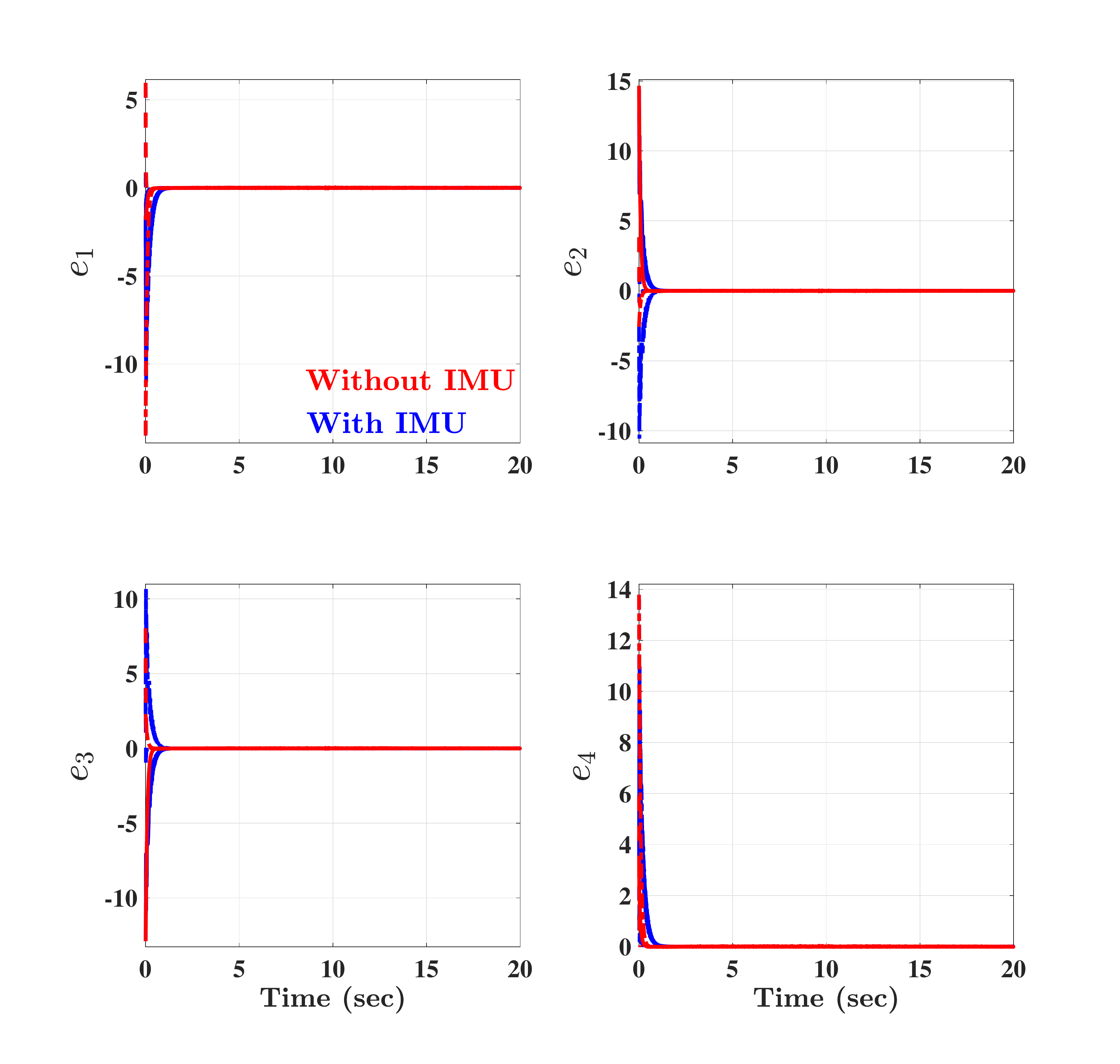}\caption{Convergence of the error trajectories of $e_{i}=[e_{i1},e_{i2},e_{i3}]^{\top}$for
		$i=1,2,3,4$ used in the Lyapunov function candidate. The proposed
		nonlinear filter for SLAM with IMU described in Subsection \ref{subsec:Det_with_IMU}
		is plotted in blue against the nonlinear filter for SLAM described
		in Subsection \ref{subsec:Det_without_IMU} plotted in red.}
	\label{fig:SLAM_error1}
\end{figure}

Fig. \ref{fig:SLAM_error1} summaries the asymptotic convergence behavior
of $e_{i}$ when using the nonlinear SLAM filter with and without
IMU for all $i=1,2,3,4$. Recall that $\tilde{R}=\hat{R}R^{\top}$,
$\tilde{P}=\hat{P}-\tilde{R}P$, $\tilde{{\rm p}}_{i}=\hat{{\rm p}}_{i}-\tilde{R}{\rm p}_{i}$,
and $||\tilde{R}||_{{\rm I}}=\frac{1}{4}{\rm Tr}\{\mathbf{I}_{3}-\tilde{R}\}$.
Considering the nonlinear SLAM filter without IMU described in Subsection
\ref{subsec:Det_without_IMU} one will notice that since $e_{i}=\tilde{{\rm p}}_{i}-\tilde{P}$,
asymptotic convergence of $e_{i}$ does not imply that $||\tilde{R}||_{{\rm I}}\rightarrow0$,
$\tilde{P}\rightarrow0$, and $\tilde{{\rm p}}_{i}\rightarrow0$.
Therefore, it follows that $\tilde{R}$, $\tilde{P}$, and $\tilde{{\rm p}}_{i}$
converge to a constant. However, the real objective of the SLAM filter
design is to achieve $||\tilde{R}||_{{\rm I}}\rightarrow0$, $||P-\hat{P}||\rightarrow0$,
and $||{\rm p}_{i}-{\rm \hat{p}}_{i}||\rightarrow0$. Fig. \ref{fig:SLAM_error2}
compares and contrasts the performance of the proposed nonlinear SLAM
filter with IMU and the nonlinear filter without IMU, emphasizing
the robustness and effectiveness in presence of IMU. As illustrated
in Fig. \ref{fig:SLAM_error2}, the nonlinear filter for SLAM without
IMU produced poor tracking performance of the error components: $||\tilde{R}||_{{\rm I}}$,
$||P-\hat{P}||$, and $||{\rm p}_{i}-{\rm \hat{p}}_{i}||$ in consistence
with \cite{zlotnik2018SLAM}. In contrast, the proposed nonlinear
filter with IMU, also depicted in Fig. \ref{fig:SLAM_error2}, demonstrates
asymptotic convergence of the attitude error ($||\tilde{R}||_{{\rm I}}$)
as well as reasonable convergence of the position error ($||P-\hat{P}||$)
and the $i$th feature error ($||{\rm p}_{i}-{\rm \hat{p}}_{i}||$)
to the close neighborhood of the origin. Despite the presence of the
residual error in $||P-\hat{P}||$ and $||{\rm p}_{i}-{\rm \hat{p}}_{i}||$,
remarkable difference is observed in the convergence of $||\tilde{R}||_{{\rm I}}$,
$||P-\hat{P}||$, and $||{\rm p}_{i}-{\rm \hat{p}}_{i}||$ between
the filter described in Subsection \ref{subsec:Det_without_IMU} and
the novel filter proposed in Subsection \ref{subsec:Det_with_IMU}.

\begin{figure}[h!]
	\centering{}\includegraphics[scale=0.29]{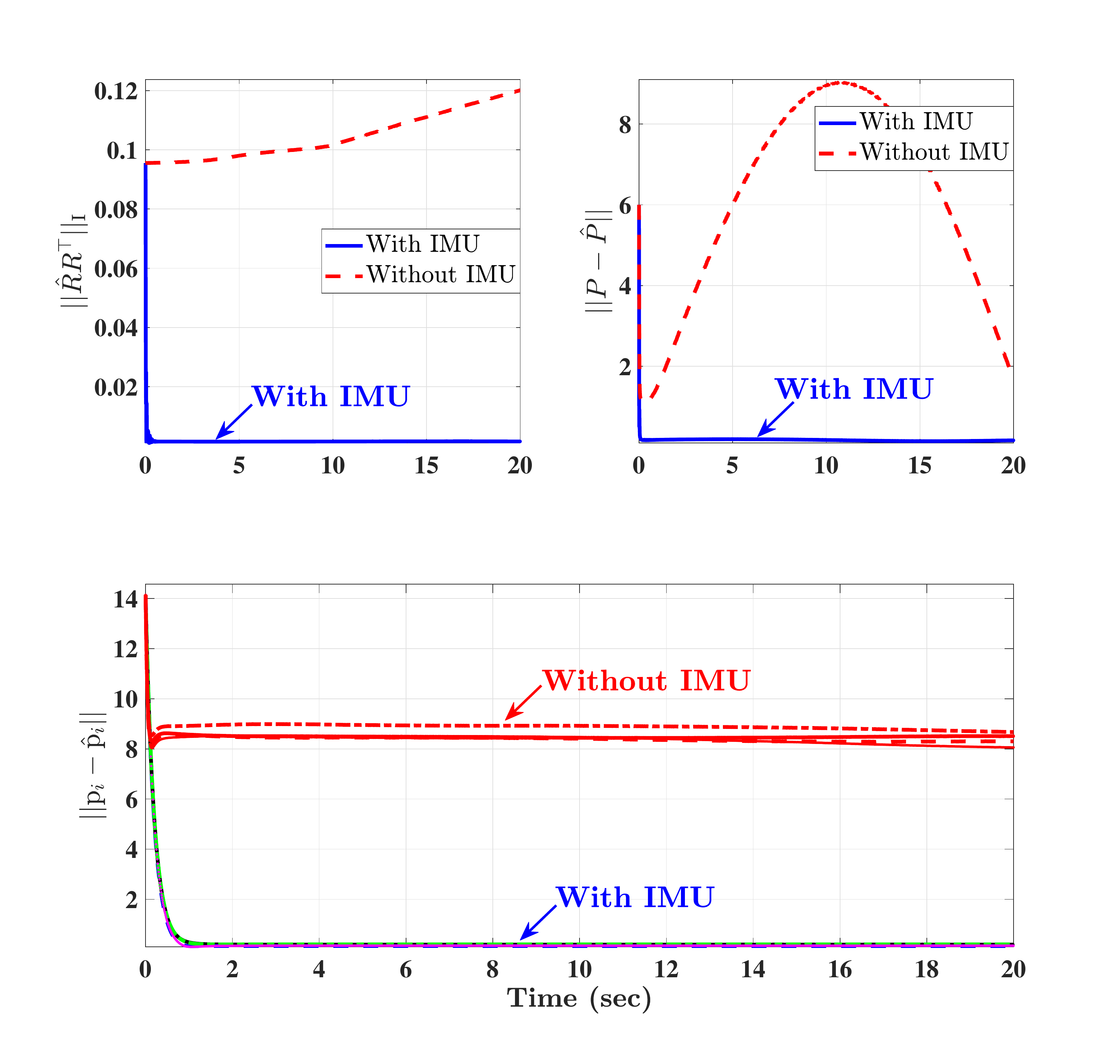}\caption{Evolution of the error trajectories of $||\hat{R}R^{\top}||_{{\rm I}}$,
		$||P-\hat{P}||$, and $||{\rm p}_{i}-{\rm \hat{p}}_{i}||$ for all
		$i=1,2,3,4$. Blue represents the proposed nonlinear filter that uses
		velocity, feature, and IMU measurements given in Subsection \ref{subsec:Det_with_IMU},
		while red corresponds to the nonlinear filter that uses velocity and
		feature measurements given in Subsection \ref{subsec:Det_without_IMU}.}
	\label{fig:SLAM_error2}
\end{figure}

While Fig. \ref{fig:SLAM_3d}, \ref{fig:SLAM_error1} and \ref{fig:SLAM_error2}
demonstrate the output performance of the proposed continuous filter
described in \eqref{eq:SLAM_T_est_dot_f2}-\eqref{eq:SLAM_p_est_dot_f2},
Fig. \ref{fig:SLAM_3d_discrete} presents its discrete counterpart
described in Algorithm \ref{alg:Alg_Disc} implemented with a sample
time of $\Delta t=0.001$ sec. The simulation of the discrete filter
utilizes the same measurements, initialization, and design parameters
introduced at the beginning of the Simulation Section with the exception
of $V=[2.5,0,0]^{\top}({\rm m/sec})$. Analogously to the continuous
filter, Fig. \ref{fig:SLAM_3d_discrete} demonstrates the superb tracking
performance of the proposed discrete nonlinear observer. In addition,
Fig. \ref{fig:SLAM_3d_discrete} reveals that the filter is computationally
cheap and can be successfully implemented using an inexpensive kit.

\begin{figure}[h!]
	\centering{}\includegraphics[scale=0.29]{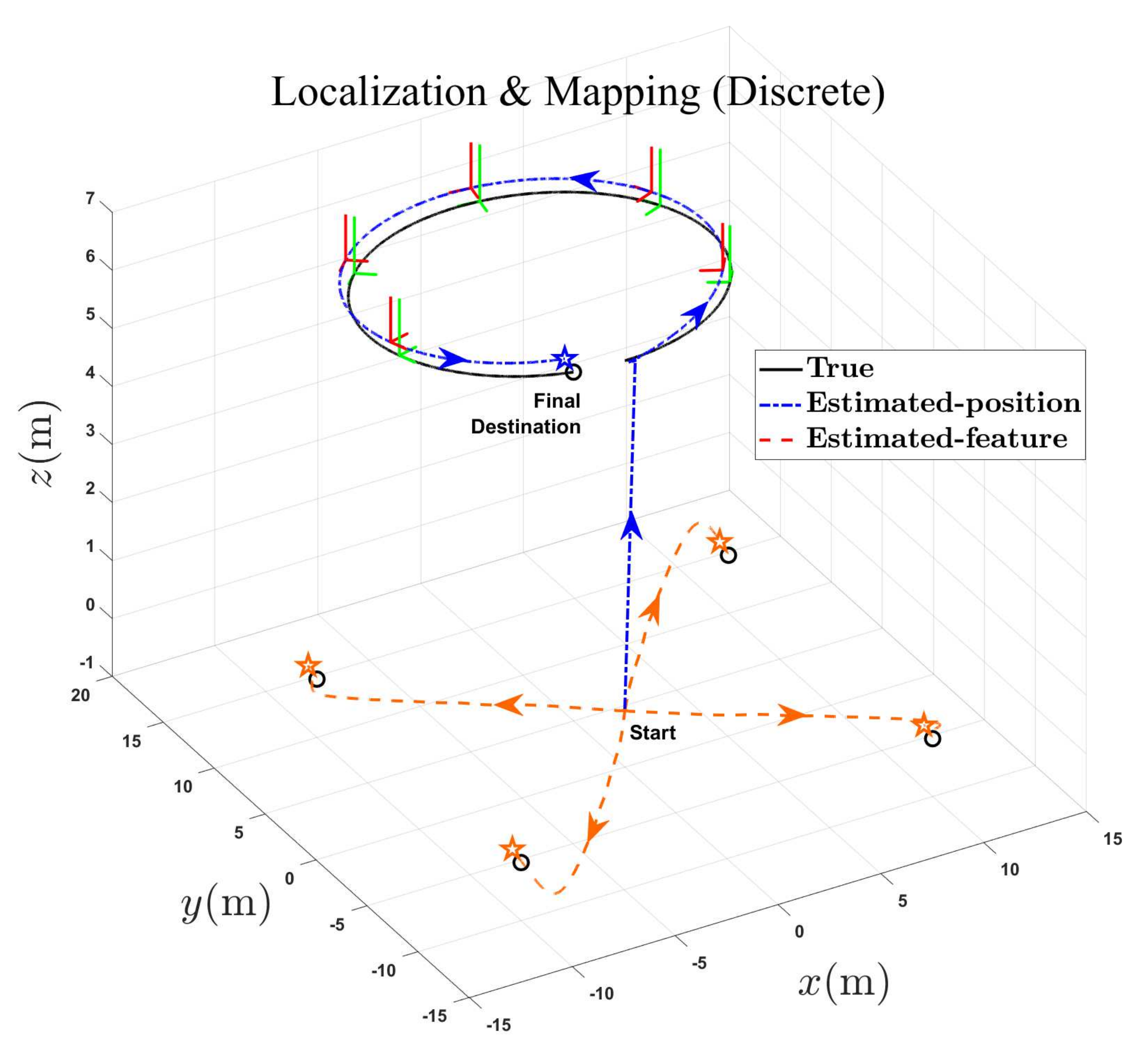}\caption{Output trajectories of the proposed nonlinear filter for SLAM described
		in Algorithm \ref{alg:Alg_Disc}. The true vehicle trajectory, and
		the vehicle and feature final destinations are plotted as a black
		solid-line, and black circle, respectively. The estimation process
		of the vehicle's trajectory, vehicle's final destination, feature
		trajectories, and feature final destinations are shown as a blue center-line,
		a blue star \textcolor{blue}{$\star$}, orange center-lines, and orange
		stars \textcolor{orange}{$\star$}, respectively. The true vehicle
		orientation is depicted as a green solid-line, while its estimation
		is plotted as a red solid-line.}
	\label{fig:SLAM_3d_discrete}
\end{figure}

To summarize, Fig. \ref{fig:SLAM_3d}, \ref{fig:SLAM_error1}, \ref{fig:SLAM_error2},
and \ref{fig:SLAM_3d_discrete} illustrate strong tracking capabilities
of the proposed filter to localize the unknown vehicle's position
and simultaneously map the unknown environment.

\section{Conclusion \label{sec:SE3_Conclusion}}

In this paper, the Simultaneous Localization and Mapping (SLAM) problem
has been addressed on the Lie group of $\mathbb{SLAM}_{n}\left(3\right)$
mimicking the nonlinear motion dynamics of the true SLAM problem.
The proposed nonlinear filter for SLAM evolved directly utilizes on
the Lie group of $\mathbb{SLAM}_{n}\left(3\right)$ utilizes the measurements
of translational and angular velocity, as well as feature and IMU
measurements. The power of the proposed approach consists in its ability
to account for the unknown bias inevitably present in velocity measurements.
As has been revealed through extensive simulation, the proposed filter
exhibits exceptional results by localizing the unknown pose of the
vehicle while simultaneously mapping the unknown environment in both
discrete and continuous time. 

\section*{Acknowledgment}

The authors would like to thank \textbf{Maria Shaposhnikova} for proofreading
the article.

\section*{Appendix\label{sec:SO3_PPF_STCH_AppendixA-1} }
\begin{center}
	\textbf{\large{}{}{}{}{}{}{}{}{}{}{}{}Quaternion Representation}{\large{}{}{}
	} 
	\par\end{center}

\noindent Let $Q=[q_{0},q^{\top}]^{\top}\in\mathbb{S}^{3}$ be a unit-quaternion
where $q_{0}\in\mathbb{R}$ and $q\in\mathbb{R}^{3}$ such that $\mathbb{S}^{3}=\{\left.Q\in\mathbb{R}^{4}\right|||Q||=\sqrt{q_{0}^{2}+q^{\top}q}=1\}$.
$Q^{-1}=[\begin{array}{cc}
q_{0} & -q^{\top}\end{array}]^{\top}\in\mathbb{S}^{3}$ stands for the inverse of $Q$. Define $\odot$ as a quaternion product,
hence, the quaternion multiplication of $Q_{1}=[\begin{array}{cc}
q_{01} & q_{1}^{\top}\end{array}]^{\top}\in\mathbb{S}^{3}$ and $Q_{2}=[\begin{array}{cc}
q_{02} & q_{2}^{\top}\end{array}]^{\top}\in\mathbb{S}^{3}$ can be represented as follows: 
\[
Q_{1}\odot Q_{2}=\left[\begin{array}{c}
q_{01}q_{02}-q_{1}^{\top}q_{2}\\
q_{01}q_{2}+q_{02}q_{1}+[q_{1}]_{\times}q_{2}
\end{array}\right]
\]
Unit-quaternion ($\mathbb{S}^{3}$) to $\mathbb{SO}\left(3\right)$
mapping can be expressed as $\mathcal{R}_{Q}:\mathbb{S}^{3}\rightarrow\mathbb{SO}\left(3\right)$
\begin{align}
\mathcal{R}_{Q} & =(q_{0}^{2}-||q||^{2})\mathbf{I}_{3}+2qq^{\top}+2q_{0}\left[q\right]_{\times}\in\mathbb{SO}\left(3\right)\label{eq:NAV_Append_SO3}
\end{align}
$Q_{{\rm I}}=[\pm1,0,0,0]^{\top}$ represents the quaternion identity
where $\mathcal{R}_{Q_{{\rm I}}}=\mathbf{I}_{3}$. More information
can be found in \cite{hashim2019AtiitudeSurvey}. Define the estimate
of $Q=[q_{0},q^{\top}]^{\top}\in\mathbb{S}^{3}$ as $\hat{Q}=[\hat{q}_{0},\hat{q}^{\top}]^{\top}\in\mathbb{S}^{3}$
with 
\[
\mathcal{R}_{\hat{Q}}=(\hat{q}_{0}^{2}-||\hat{q}||^{2})\mathbf{I}_{3}+2\hat{q}\hat{q}^{\top}+2\hat{q}_{0}\left[\hat{q}\right]_{\times}\in\mathbb{SO}\left(3\right)
\]
Recall the map in \eqref{eq:NAV_Append_SO3}. Define the map
\begin{align*}
\left[\begin{array}{c}
0\\
\mathbf{Y}(\hat{Q},x)
\end{array}\right] & =\hat{Q}\odot\left[\begin{array}{c}
0\\
x
\end{array}\right]\odot\hat{Q}^{-1}\\
\left[\begin{array}{c}
0\\
\mathbf{Y}(\hat{Q}^{-1},x)
\end{array}\right] & =\hat{Q}^{-1}\odot\left[\begin{array}{c}
0\\
x
\end{array}\right]\odot\hat{Q}
\end{align*}
with $\mathbf{Y}(\hat{Q},y_{i})\in\mathbb{R}^{3}$, $x\in\mathbb{R}^{3}$,
and $\hat{Q}\in\mathbb{S}^{3}$. Let us reformulate the observer in
\eqref{eq:SLAM_T_est_dot_f2}, \eqref{eq:SLAM_p_est_dot_f2}, \eqref{eq:SLAM_b_est_dot_f2},
and \eqref{eq:SLAM_W_f2} along with its implementation steps in terms
of unit-quaternion: 
\[
\begin{cases}
e_{i} & =\hat{{\rm p}}_{i}-\mathbf{Y}\left(\hat{Q},y_{i}\right)-\hat{P},\hspace{1em}i=1,2,\ldots,n\\
\boldsymbol{\Upsilon}(\tilde{R}M) & =\mathbf{Y}\left(\hat{Q},\sum_{j=1}^{n_{{\rm R}}}\left(\frac{s_{j}}{2}\hat{\upsilon}_{j}^{a}\times\upsilon_{j}^{a}\right)\right)\\
\tau_{R} & =\underline{\lambda}(\breve{\mathbf{M}})\times(1+\pi(\tilde{R},M))\\
\chi & =\Omega_{m}-\hat{b}_{\Omega}-W_{\Omega}\\
\dot{\hat{Q}} & =\frac{1}{2}\left[\begin{array}{cc}
0 & -\chi^{\top}\\
\chi & -\left[\chi\right]_{\times}
\end{array}\right]\hat{Q},\hspace{1em}\hat{Q}(0)=Q_{{\rm I}}\\
\dot{\hat{P}} & =\mathbf{Y}\left(\hat{Q},V_{m}-\hat{b}_{V}-W_{V}\right)\\
\dot{{\rm \hat{p}}}_{i} & =-k_{1}e_{i}+\mathbf{Y}\left(\hat{Q},[y_{i}]_{\times}W_{\Omega}\right)\\
\dot{\hat{b}}_{\Omega} & =\frac{\Gamma_{1}}{2}\mathbf{Y}\left(\hat{Q}^{-1},\boldsymbol{\Upsilon}(\tilde{R}M)\right)\\
& \hspace{1em}-\sum_{i=1}^{n}\frac{\Gamma_{1}}{\alpha_{i}}[y_{i}]_{\times}\mathbf{Y}\left(\hat{Q}^{-1},e_{i}\right)\\
\dot{\hat{b}}_{V} & =-\sum_{i=1}^{n}\frac{\Gamma_{2}}{\alpha_{i}}\mathbf{Y}\left(\hat{Q}^{-1},e_{i}\right)\\
W_{\Omega} & =\frac{k_{w}}{\tau_{R}}\mathbf{Y}\left(\hat{Q}^{-1},\boldsymbol{\Upsilon}(\tilde{R}M)\right)\\
W_{V} & =-\sum_{i=1}^{n}\frac{k_{2}}{\alpha_{i}}\mathbf{Y}\left(\hat{Q}^{-1},e_{i}\right)
\end{cases}
\]

\bibliographystyle{IEEEtran}
\bibliography{bib_SLAM}

\begin{thebibliography}{10}
\providecommand{\url}[1]{#1}
\csname url@samestyle\endcsname
\providecommand{\newblock}{\relax}
\providecommand{\bibinfo}[2]{#2}
\providecommand{\BIBentrySTDinterwordspacing}{\spaceskip=0pt\relax}
\providecommand{\BIBentryALTinterwordstretchfactor}{4}
\providecommand{\BIBentryALTinterwordspacing}{\spaceskip=\fontdimen2\font plus
\BIBentryALTinterwordstretchfactor\fontdimen3\font minus
  \fontdimen4\font\relax}
\providecommand{\BIBforeignlanguage}[2]{{%
\expandafter\ifx\csname l@#1\endcsname\relax
\typeout{** WARNING: IEEEtran.bst: No hyphenation pattern has been}%
\typeout{** loaded for the language `#1'. Using the pattern for}%
\typeout{** the default language instead.}%
\else
\language=\csname l@#1\endcsname
\fi
#2}}
\providecommand{\BIBdecl}{\relax}
\BIBdecl

\bibitem{thrun2002robotic}
S.~Thrun \emph{et~al.}, ``Robotic mapping: A survey,'' \emph{Exploring
  artificial intelligence in the new millennium}, vol.~1, no. 1-35, p.~1, 2002.

\bibitem{hashim2019SE3Det}
H.~A. Hashim, L.~J. Brown, and K.~McIsaac, ``Nonlinear pose filters on the
  special euclidean group {SE}(3) with guaranteed transient and steady-state
  performance,'' \emph{IEEE Transactions on Systems, Man, and Cybernetics:
  Systems}, pp. 1--14, 2019.

\bibitem{zlotnik2018higher}
D.~E. Zlotnik and J.~R. Forbes, ``Higher order nonlinear complementary
  filtering on lie groups,'' \emph{IEEE Transactions on Automatic Control},
  vol.~64, no.~5, pp. 1772--1783, 2018.

\bibitem{hashim2020SE3Stochastic}
H.~A. Hashim and F.~L. Lewis, ``Nonlinear stochastic estimators on the special
  euclidean group {SE}(3) using uncertain imu and vision measurements,''
  \emph{IEEE Transactions on Systems, Man, and Cybernetics: Systems}, vol.~PP,
  no.~PP, pp. PP--PP, 2020.

\bibitem{choset2000SLAM}
H.~Choset, S.~Walker, K.~Eiamsa-Ard, and J.~Burdick, ``Sensor-based
  exploration: Incremental construction of the hierarchical generalized voronoi
  graph,'' \emph{The International Journal of Robotics Research}, vol.~19,
  no.~2, pp. 126--148, 2000.

\bibitem{durrant2006simultaneous}
H.~Durrant-Whyte and T.~Bailey, ``Simultaneous localization and mapping: part
  i,'' \emph{IEEE robotics \& automation magazine}, vol.~13, no.~2, pp.
  99--110, 2006.

\bibitem{bekris2006evaluation}
K.~E. Bekris, M.~Glick, and L.~E. Kavraki, ``Evaluation of algorithms for
  bearing-only slam,'' in \emph{Proceedings 2006 IEEE International Conference
  on Robotics and Automation, 2006. ICRA 2006.}\hskip 1em plus 0.5em minus
  0.4em\relax IEEE, 2006, pp. 1937--1943.

\bibitem{davison2007monoslam}
A.~J. Davison, I.~D. Reid, N.~D. Molton, and O.~Stasse, ``Monoslam: Real-time
  single camera slam,'' \emph{IEEE transactions on pattern analysis and machine
  intelligence}, vol.~29, no.~6, pp. 1052--1067, 2007.

\bibitem{zlotnik2018SLAM}
D.~E. Zlotnik and J.~R. Forbes, ``Gradient-based observer for simultaneous
  localization and mapping,'' \emph{IEEE Transactions on Automatic Control},
  vol.~63, no.~12, pp. 4338--4344, 2018.

\bibitem{liu2018brain}
Y.~Liu, Z.~Li, T.~Zhang, and S.~Zhao, ``Brain-robot interface-based navigation
  control of a mobile robot in corridor environments,'' \emph{IEEE Transactions
  on Systems, Man, and Cybernetics: Systems}, 2018.

\bibitem{maurovic2017path}
I.~Maurovic, M.~Seder, K.~Lenac, and I.~Petrovic, ``Path planning for active
  slam based on the d* algorithm with negative edge weights,'' \emph{IEEE
  Transactions on Systems, Man, and Cybernetics: Systems}, vol.~48, no.~8, pp.
  1321--1331, 2017.

\bibitem{hashim2020SLAMLetter}
H.~A. Hashim, ``Guaranteed performance nonlinear observer for simultaneous
  localization and mapping,'' \emph{IEEE Control Systems Letters}, vol.~5,
  no.~1, pp. 91--96, 2021.

\bibitem{yuan2019multisensor}
W.~Yuan, Z.~Li, and C.-Y. Su, ``Multisensor-based navigation and control of a
  mobile service robot,'' \emph{IEEE Transactions on Systems, Man, and
  Cybernetics: Systems}, 2019.

\bibitem{montemerlo2007fastslam}
M.~Montemerlo and S.~Thrun, \emph{FastSLAM: A scalable method for the
  simultaneous localization and mapping problem in robotics}.\hskip 1em plus
  0.5em minus 0.4em\relax Springer, 2007, vol.~27.

\bibitem{kaess2008isam}
M.~Kaess, A.~Ranganathan, and F.~Dellaert, ``isam: Incremental smoothing and
  mapping,'' \emph{IEEE Transactions on Robotics}, vol.~24, no.~6, pp.
  1365--1378, 2008.

\bibitem{huang2007convergence}
S.~Huang and G.~Dissanayake, ``Convergence and consistency analysis for
  extended kalman filter based slam,'' \emph{IEEE Transactions on robotics},
  vol.~23, no.~5, pp. 1036--1049, 2007.

\bibitem{chatterjee2007neuro}
A.~Chatterjee and F.~Matsuno, ``A neuro-fuzzy assisted extended kalman
  filter-based approach for simultaneous localization and mapping (slam)
  problems,'' \emph{IEEE transactions on fuzzy systems}, vol.~15, no.~5, pp.
  984--997, 2007.

\bibitem{zhang2017EKF_SLAM}
T.~Zhang, K.~Wu, J.~Song, S.~Huang, and G.~Dissanayake, ``Convergence and
  consistency analysis for a 3-d invariant-ekf slam,'' \emph{IEEE Robotics and
  Automation Letters}, vol.~2, no.~2, pp. 733--740, 2017.

\bibitem{dissanayake2011review}
G.~Dissanayake, S.~Huang, Z.~Wang, and R.~Ranasinghe, ``A review of recent
  developments in simultaneous localization and mapping,'' in \emph{2011 6th
  International Conference on Industrial and Information Systems}.\hskip 1em
  plus 0.5em minus 0.4em\relax IEEE, 2011, pp. 477--482.

\bibitem{cadena2016past}
C.~Cadena, L.~Carlone, H.~Carrillo, Y.~Latif, D.~Scaramuzza, J.~Neira, I.~Reid,
  and J.~J. Leonard, ``Past, present, and future of simultaneous localization
  and mapping: Toward the robust-perception age,'' \emph{IEEE Transactions on
  robotics}, vol.~32, no.~6, pp. 1309--1332, 2016.

\bibitem{lee2012exponential}
T.~Lee, ``Exponential stability of an attitude tracking control system on so
  (3) for large-angle rotational maneuvers,'' \emph{Systems \& Control
  Letters}, vol.~61, no.~1, pp. 231--237, 2012.

\bibitem{grip2012attitude}
H.~F. Grip, T.~I. Fossen, T.~A. Johansen, and A.~Saberi, ``Attitude estimation
  using biased gyro and vector measurements with time-varying reference
  vectors,'' \emph{IEEE Transactions on Automatic Control}, vol.~57, no.~5, pp.
  1332--1338, 2012.

\bibitem{hashim2018SO3Stochastic}
H.~A. Hashim, L.~J. Brown, and K.~McIsaac, ``Nonlinear stochastic attitude
  filters on the special orthogonal group 3: Ito and stratonovich,'' \emph{IEEE
  Transactions on Systems, Man, and Cybernetics: Systems}, vol.~49, no.~9, pp.
  1853--1865, 2019.

\bibitem{hashim2019SO3Wiley}
H.~A. Hashim, ``Systematic convergence of nonlinear stochastic estimators on
  the special orthogonal group {SO}(3),'' \emph{International Journal of Robust
  and Nonlinear Control}, vol.~30, no.~10, pp. 3848--3870, 2020.

\bibitem{strasdat2012local}
H.~Strasdat, ``Local accuracy and global consistency for efficient visual
  slam,'' Ph.D. dissertation, Department of Computing, Imperial College London,
  2012.

\bibitem{johansen2016globally}
T.~A. Johansen and E.~Brekke, ``Globally exponentially stable kalman filtering
  for slam with ahrs,'' in \emph{2016 19th International Conference on
  Information Fusion (FUSION)}.\hskip 1em plus 0.5em minus 0.4em\relax IEEE,
  2016, pp. 909--916.

\bibitem{hashim2020AtiitudeSurvey}
H.~A. Hashim, ``Attitude determination and estimation using vector
  observations: Review, challenges and comparative results,'' \emph{ar{X}iv
  preprint ar{X}iv:2001.03787}, 2020.

\bibitem{lee2006SLAM_observability}
K.~W. Lee, W.~S. Wijesoma, and J.~I. Guzman, ``On the observability and
  observability analysis of slam,'' in \emph{2006 IEEE/RSJ International
  Conference on Intelligent Robots and Systems}.\hskip 1em plus 0.5em minus
  0.4em\relax IEEE, 2006, pp. 3569--3574.

\bibitem{bullo2004geometric}
F.~Bullo and A.~D. Lewis, \emph{Geometric control of mechanical systems:
  modeling, analysis, and design for simple mechanical control systems}.\hskip
  1em plus 0.5em minus 0.4em\relax Springer Science \& Business Media, 2004,
  vol.~49.

\bibitem{hashim2019AtiitudeSurvey}
H.~A. Hashim, ``Special orthogonal group {SO}(3), euler angles, angle-axis,
  rodriguez vector and unit-quaternion: Overview, mapping and challenges,''
  \emph{ar{X}iv preprint ar{X}iv:1909.06669}, 2019.

\end{thebibliography}

\vspace{310pt}

\section*{AUTHOR INFORMATION}
\vspace{10pt}

	{\bf Hashim A. Hashim} (Member, IEEE) is an Assistant Professor with the Department of Engineering and Applied Science, Thompson Rivers University, Kamloops, British Columbia, Canada. He received the B.Sc. degree in Mechatronics, Department of Mechanical Engineering from Helwan University, Cairo, Egypt, the M.Sc. in Systems and Control Engineering, Department of Systems Engineering from King Fahd University of Petroleum \& Minerals, Dhahran, Saudi Arabia, and the Ph.D. in Robotics and Control, Department of Electrical and Computer Engineering at Western University, Ontario, Canada.\\
	His current research interests include stochastic and deterministic attitude and pose filters, Guidance, navigation and control, simultaneous localization and mapping, control of multi-agent systems, and optimization techniques.
	
\underline{Contact Information}: \href{mailto:hhashim@tru.ca}{hhashim@tru.ca}.
\vspace{50pt}

	{\bf Abdelrahman E.E. Eltoukhy} received his BSc Degree in Production Engineering from Helwan University, Egypt, and obtained his MSc in Engineering and Management from the Politecnico Di Torino, Italy. He obtained his PhD degree from The Hong Kong Polytechnic University, Hong Kong. He is currently a Research Assistant Professor in Industrial and Systems Engineering department, The Hong Kong Polytechnic University, Hong Kong.\\
	His current research interests include airline schedule planning, logistics and supply chain management, operations research, and simulation.

\end{document}